\documentclass[11pt]{article}
\usepackage{fullpage}
\usepackage{graphicx,verbatim,color}
\usepackage{amsmath}
\usepackage{amssymb}
\usepackage{comment}
\usepackage{algorithm}
\usepackage{algorithmic}
\usepackage{xcolor}
\usepackage[utf8]{inputenc} 
\usepackage[T1]{fontenc}    
\usepackage{url}            
\usepackage{booktabs}       
\usepackage{amsfonts}       
\usepackage{nicefrac}       
\usepackage{microtype}      
\usepackage{graphicx}
\usepackage{subfig}
\usepackage{algorithm}
 \usepackage{multirow}
\usepackage{epstopdf}

\definecolor{darkgreen}{RGB}{0,150,0}
\definecolor{darkblue}{RGB}{0,0,200}
\definecolor{darkgreen}{RGB}{0,150,0}
\definecolor{darkblue}{RGB}{0,0,200}

\newcommand{\green}[1]{{\color{darkgreen}{#1}}}
\newcommand{\ols}{OLS}
\usepackage{mathrsfs}

\usepackage{amsthm}\usepackage{dsfont}\usepackage{array}\usepackage{mathrsfs}
\makeatletter
\def\BState{\State\hskip-\ALG@thistlm}
\makeatother
\usepackage[utf8]{inputenc}

\usepackage{enumitem}
\usepackage{bm}

\newcommand{\mtx}[1]{\bm{#1}}
\newcommand{\cG}{\mathcal{G}}
\newcommand{\cH}{\mathcal{H}}
\newcommand{\cP}{\mathcal{P}}

\newcommand{\beq}{\begin{equation}}
\newcommand{\eeq}{\end{equation}}

\newcommand{\Ub}{{\mtx{U}}}

\newcommand{\diag}[1]{\text{diag}(#1)}

\newcommand{\Hc}{{\cal{H}}}

\newcommand{\order}[1]{{\cal{O}}(#1)}

\newcommand{\Nn}{\mathcal{N}}

\definecolor{darkred}{RGB}{150,0,0}

\newcommand{\snr}{\textbf{snr}}

\definecolor{emmanuel}{RGB}{255,127,0}

\newcommand{\R}{\mathbb{R}}

\newcommand{\E}{\operatorname{\mathbb{E}}}

\newcommand{\rank}{\operatorname{rank}}

\newcommand{\ones}{\mathbf{1}}
\newcommand{\zeros}{\mathbf{0}}
\newcommand{\argmin}{\mathrm{argmin}}

\newcommand{\Ttot}{T_{\mathrm{tot}}}
\newcommand{\ntune}{n_{\mathrm{t}}}
\newcommand{\red}[1]{{\color{darkblue}{#1}}}

\title{\LARGE{\bf System Identification via Nuclear Norm Regularization}} 

\author{%
 Yue Sun \and Samet Oymak \and Maryam Fazel
}
\begin{document}
\maketitle
\theoremstyle{plain}
\newtheorem{lemma}{\textbf{Lemma}}
\newtheorem{theorem}{\textbf{Theorem}}
\newtheorem{corollary}{\textbf{Corollary}}
\newtheorem{assumption}{\textbf{Assumption}}
\newtheorem{example}{\textbf{Example}}
\newtheorem{definition}{\textbf{Definition}}
\newtheorem{conjecture}{\textbf{Conjecture}}
\newtheorem{claim}{\textbf{Claim}}
\newtheorem{question}{\textbf{Question}}
\theoremstyle{definition}

\theoremstyle{remark}
\newtheorem{remark}{\textbf{Remark}}

\begin{abstract}

 This paper studies the problem of identifying low-order linear systems via Hankel nuclear norm regularization. Hankel regularization encourages the low-rankness of the Hankel matrix, which maps to the low-orderness of the system. We provide novel statistical analysis for this regularization and carefully contrast it with the unregularized ordinary least-squares (\ols) estimator.
 
   Our analysis leads to new bounds on estimating the impulse response and the Hankel matrix associated with the linear system. We first design an input excitation and show that Hankel regularization enables one to recover the system using optimal number of observations in the true system order and achieve strong statistical estimation rates. Surprisingly, we demonstrate that the input design indeed matters, by showing that intuitive choices such as i.i.d.~Gaussian input leads to provably sub-optimal sample complexity. 
   
   To better understand the benefits of regularization, we also revisit the \ols~estimator. Besides refining existing bounds, we experimentally identify when regularized approach improves over \ols: (1) For low-order systems with slow impulse-response decay, \ols~method performs poorly in terms of sample complexity, (2)  Hankel matrix returned by regularization has a more clear singular value gap that ease identification of the system order, (3) Hankel regularization is less sensitive to hyperparameter choice. Finally, we establish model selection guarantees through a joint train-validation procedure where we tune the regularization parameter for near-optimal estimation.
   \end{abstract}

\section{Introduction}
System identification is an important topic in control theory. 
Accurate estimation of system dynamics is the basis of control or policy decision problems in tasks varying from linear-quadratic control to deep reinforcement learning. 
Consider a linear time-invariant system of order $R$ with the \emph{minimal} state-space representation
\begin{equation}\label{eq:linear_system}
    \begin{split}
        x_{t+1} &= Ax_t + Bu_t,\\
        y_t &= Cx_t + D u_t+ z_t,
    \end{split}
\end{equation}
where $x_t\in \R^{R}$ is the state, $u_t\in \R^{p}$ is the input, $y_t\in \R^{m}$ is the output, 
$z_t\in\R^m$ is the output noise, 
$A\in\R^{R\times R}$, $B\in \R^{R\times p}$, $C\in\R^{m\times R}$, $D\in\R^{m\times p}$ 
are the system parameters, and $x_0$ is the initial state (in this paper, we assume $x_0=0$). Generally with the same input and output, the dimension of the hidden state $x$ can be any number no less than $R$, and we are interested in the minimum dimensional representation (i.e., minimal realization) in this paper.

\emph{The goal of system identification is to find the system parameters, such as $A,B,C,D$ matrices or impulse response, given input and output observations.} If $(C,D)=(I,0)$, we directly observe the state. A notable line of work derives statistical bounds for system identification with limited \emph{state} observations from a single output trajectory (defined in Fig. \ref{fig:data_acq}) with a random input \cite{abbasi2011regret, simchowitz2018learning,sarkar2019near}. The state evolves as $x_{t+1} = Ax_t + \eta_t$ where $\eta_t$ is the white noise that provides excitation to states \cite{simchowitz2018learning,sarkar2019near}. They recover $A$ by solving a least-squares problem. The main proof approach comes from an analysis of martingales \cite[Thm 2,3]{abbasi2011online}. \cite{simchowitz2018learning} assumes that the system is stable whereas \cite{sarkar2019near} removes the assumptions on the spectral radius of $A$.

When we do not directly observe the state $x$ (also known as hidden-state), one has only access to $u_t$ and $y_t$ and lack the full information on $x_t$. We recover the impulse response (also known as the Markov parameters) sequence $h_0=D$, $h_t=CA^{t-1}B\in\R^{m\times p}$ for $t=1,2,\ldots$ that uniquely identifies the end-to-end behavior of the system. The impulse response can have infinite length, and we let $h = [D, CB, CAB, CA^2B,\ldots, CA^{2n-3}B]^\top$ denote its first $2n-1$ entries, which can be later placed into an $n\times n$ Hankel matrix. Without knowing the system order, we consider recovering the first $n$ terms of $h$ where $n$ is larger than system order $R$. 
To this end, let us also define the Hankel map $\cH:\R^{m\times (2n-1)p}\rightarrow\R^{mn\times pn}$ as 
\begin{align}
    H:=\cH(h) = \begin{bmatrix}
    h_1 & h_2 & ... &h_n\\
    h_2 & h_3 & ... & h_{n+1}\\
    ...\\
    h_n & h_{n+1} & ... & h_{2n-1}
    \end{bmatrix}.\label{eq:H}
\end{align}
If $n\ge R$, the Hankel matrix $H$ is of rank $R$ regardless of $n$  \cite[Sec. 5.5]{sontag2013mathematical}. Specifically, we will assume that $R$ is small, so the Hankel matrix is low rank. Our goal is to recover a low rank Hankel matrix. It is known that nuclear norm regularization is used to find a low rank matrix \cite{recht2010guaranteed,fazel2001rank}, and \cite{fazel2003matrix} uses it for recovering a low rank Hankel matrix. 

Low-rank Hankel matrices arise in a range of applications, from dynamical systems -- where the rank corresponds to a low order or MacMillan degree for the system \cite{sontag2013mathematical, fazel2003matrix} -- to signal processing problems. The latter includes recovering sum of complex exponentials \cite{cai2016robust, xu2018sep} (where the rank of the Hankel matrix is the number of summands), shape-from-moments estimation in tomography and geophysical inversion \cite{elad2004shape} (where the vertices of an object are probed and the output is a sum of exponentials), and video in-painting \cite{ding2007rank} (where the video is regarded as a low order system).

\noindent\textbf{Performance criteria for system identification:} To explain our contributions, we introduce common performance metrics. Refs.~\cite{oymak2018non} and \cite{sarkar2019finite} recover the system from single rollout/trajectory ("rollout" is defined in Sec.~\ref{s:formulation}) of the input signal, whereas our work, \cite{tu2017non} and \cite{cai2016robust} require multiple rollouts. 
To ensure a standardized comparison, we define \emph{sample complexity} to be the number of equations (equality constraints in variables $h_t$) used in the problem formulation, which is same as the number of observed outputs (see Fig.~\ref{fig:data_acq} and Sec.~\ref{s:formulation}). 
With this, we explore the following performance metrics for learning the system from $T$ output measurements.
\begin{itemize}[leftmargin=*]
\setlength{\itemsep}{0pt}
\setlength{\parskip}{0pt}
\setlength{\parsep}{0pt}
    \item {\bf{Sample complexity:}} The minimum sample size $T$ for recovering system parameters with zero error when the noise is set to $z=0$. This quantity is lower bounded by the system order. 
    System order can be seen as the ``degrees of freedom" of the system.
    \item {\bf{Impulse Response (IR) Estimation Error:}} The Frobenius norm error $\|\hat h - h\|_F$ for the IR. A good estimate of IR enables the accurate prediction of the system output.
    \item {\bf{Hankel Estimation Error:}} The spectral norm error $\|\cH(\hat h - h)\|$ of the Hankel matrix. This metric is particularly important for system identification as described below.
\end{itemize}

The Hankel spectral norm error is a critical quantity for several reasons. First, the Hankel spectral norm error connects to the $\cal{H}_\infty$ estimation of the system \cite{sanchez1998robust}. 
Secondly, bounding this error allows for robustly finding balanced realizations of the system; for example, the error in reconstructing state-space matrices ($A,B,C,D$) via the Ho-Kalman procedure is bounded by the Hankel spectral error. Finally, it is beneficial in model selection, as a small spectral error helps distinguish the true singular values of the system from the spurious ones caused by estimation error. Indeed, as illustrated in the experiments, the Hankel singular value gap of the solution of the regularized algorithm is more visible compared to least-squares, which aids in identifying the true order of the system as explored in Sec. \ref{s:experiments}.

\noindent\textbf{Algorithms: Hankel-regularization \& OLS.} In our analysis, we consider a multiple rollout setup where we measure the system dynamics with $T$ separate rollouts. For each rollout, the input sequence is $u^{(i)}=[u_{2n-1}^{(i)},...,u_{1}^{(i)}]\in\R^{(2n-1)p}$ and we measure the system output at time $2n-1$. Note that the $i^{th}$ output at time $2n-1$ is simply $h^\top u^{(i)}$. Define $\bar \Ub\in\R^{T\times (2n-1)p}$ where the $i^{th}$ row is $u^{(i)}$. Let $y\in\R^{T\times m}$ denote the corresponding observed outputs. Hankel-regularization refers to the nuclear norm regularized problem \eqref{eq:lasso_prob_h}. 
\begin{align}\label{eq:lasso_prob_h}
    \hat{h}=\arg\min_{h'}&\quad  \frac{1}{2}\|\bar \Ub h' - y\|_F^2 + \lambda\|\mathcal{H}(h')\|_*,\tag{HNN}
\end{align}
Finally, setting $\lambda=0$, we obtain the special case of ordinary least-squares (OLS).

\section{Contributions}
Our main contribution is establishing data-driven guarantees for Hankel nuclear norm regularization and shedding light on the benefit of regularization through a comparison to the ordinary least-squares (\ols) estimator. Specifically, a summary of our findings are as follows.

\noindent$\bullet$ {\bf{Hankel nuclear norm}} (Sec.~\ref{s:reg} \& \ref{sec:no_weight}): For multi-input/single-output (MISO) systems ($p$ input channels), we establish \emph{near-optimal sample complexity} bounds for the Hankel-regularized system identification, showing the required sample size grows as ${\cal{O}}(pR\log^2n)$ where $R$ is the system order and $n$ is the Hankel size. This result utilizes an \emph{input-shaping} strategy (rather than i.i.d.~excitation, see Fig.~\ref{figshape}) and builds on \cite{cai2016robust} who studied the recovery of a sum-of-exponentials signal. Our bound significantly improves over naive bounds. For instance, without Hankel structure, enforcing low-rank would require ${\cal{O}}(nR)$ samples and enforcing Hankel structure without low-rank would require ${\cal{O}}(n)$ samples.

We also establish finite sample bounds on the IR and Hankel spectral errors. Our rates are on par with the \ols~rates; however, unlike \ols, they also apply in the small sample size regime $pn\gtrsim T\gtrsim pR\log^2 n$.

Surprisingly, Sec.~\ref{sec:no_weight} shows that the \emph{input-shaping} is necessary for the logarithmic sample complexity in $n$. Specifically, we prove that if the inputs are i.i.d.~standard normal (Fig. \ref{fignoshape}), the minimum number of observations to exactly recover the impulse response in the noiseless case grows as $T\gtrsim n^{1/6}$. 

\begin{figure}[t!]
\centering
{
\subfloat[]{
\label{figshape}
\includegraphics[width=0.46\textwidth]{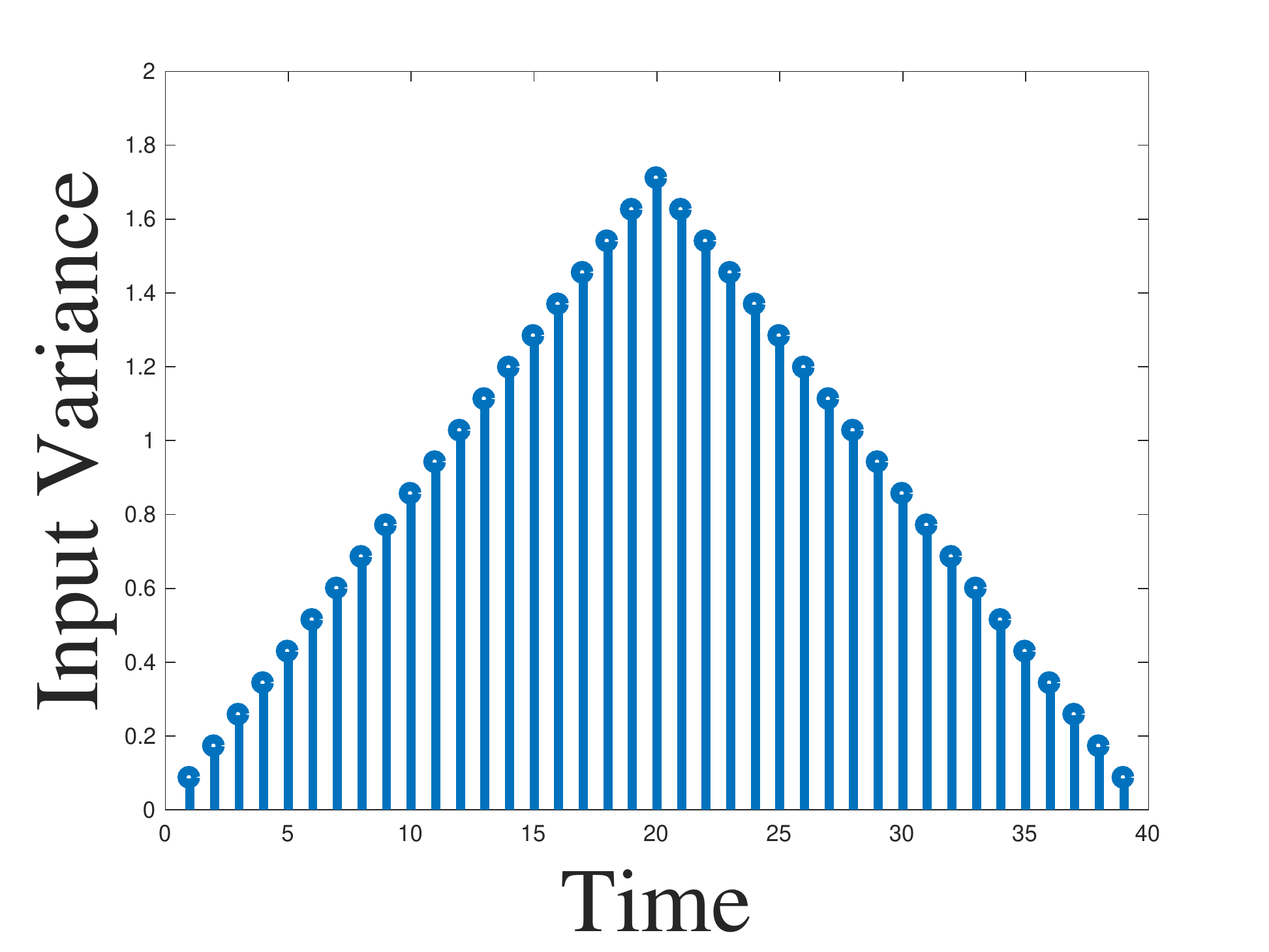}
}
\subfloat[]{
\label{fignoshape}
\includegraphics[width=0.46\textwidth]{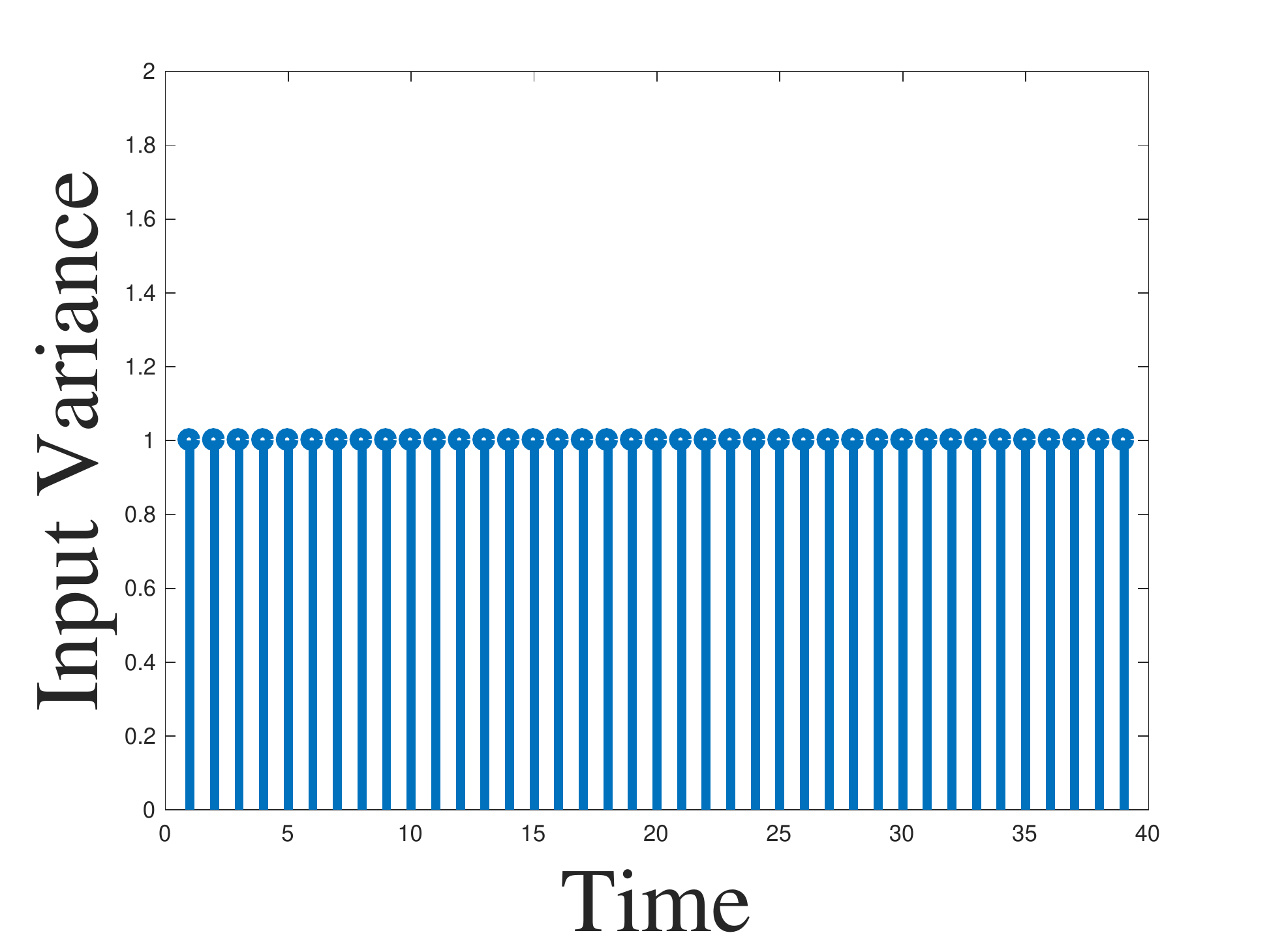}
}
}
{\caption{(a) Shaped input (where variance of $u(t)$ changes over time): recovery is guaranteed when $T \approx R$; (b) i.i.d input (fixed variance): recovery fails with high probability when $T\lesssim n^{1/6}$. See Sec. \ref{sec:no_weight}.}\label{fig:shape}}
\end{figure}

\noindent$\bullet$ {\bf{Sharpening \ols~bounds}} (Sec. \ref{s:ls}): For multi-input/multi-output (MIMO) systems, we establish a \emph{near-optimal spectral error rate} for the Hankel matrix when $T\gtrsim np$. Our error rate improves over that of \cite{oymak2018non} and our sample complexity improves over \cite{sarkar2019finite} and \cite{tu2017non} which require ${\cal{O}}(n^2)$ samples rather than ${\cal{O}}(n)$. This refinement is accomplished by relating the IR and Hankel errors. Specifically, using the fact that rows of the Hankel matrix are subsets of the IR sequence, we always have the inequality 
\begin{align}
\|\hat h - h\|_F/\sqrt{2}\leq \|\cH(\hat h - h)\|\leq \sqrt{n}\|\hat h - h\|_F.\label{spec relate}
\end{align}
Observe that there is a factor of $\sqrt{n}$ gap between the left-hand and right-hand side inequalities. We show that the left-hand side is typically the tighter one, thus $\|\hat h - h\|_F\sim \|\cH(\hat h - h)\|$.

\noindent$\bullet$ \textbf{Guarantees on accurate model-selection (Sec. \ref{s:model_selection}):} The Hankel-regularized algorithm requires a proper choice of the regularization parameter $\lambda$. In practice, the optimal choice is data dependent and one usually estimates $\lambda$ via trial and error based on the validation error. We provide a complete procedure for model selection (training \& validation phases), and establish statistical guarantees for it.

\noindent$\bullet$ {\bf{Contrasting Hankel regularization and \ols~(Sec.~\ref{s:experiments}):}} 
Finally, we assess the benefits of regularization via numerical experiments on system identification focusing on data collected from a single-trajectory.

We first consider synthetic data and focus on low-order systems with slow impulse-response decay. The slow-decay is intended to exacerbate the FIR approximation error arising from truncating the impulse-response at $2n-1$ terms. In this setting, \ols~as well as \cite{sarkar2019finite} are shown to perform poorly. In constrast, Hankel-regularization better avoids the truncation error as it allows for fitting a long impulse-response with few data (due to logarithmic dependence on $n$).

Our real-data experiments (on a low-order example from the DaISy datasets \cite{de1997daisy}) suggest that the regularized algorithm has empirical benefits in sample complexity, estimation error, and Hankel spectral gap, and demonstrate that the regularized algorithm is less sensitive to the choice of the tuning parameter, compared to \ols~whose tuning parameter is the Hankel size $n$. 
Finally, comparison of least-squares approaches in \cite{oymak2018non} (\ols) and \cite{sarkar2019finite} reveals that \ols~(which directly estimates the impulse response)
performs substantially better than the latter (which estimates the Hankel matrix). This highlights the role of proper parameterization in system identification. 

\subsection{Prior Art}
The traditional unregularized methods include Cadzow approach \cite{cadzow1988signal,gillard2010cadzow}, matrix pencil method \cite{sarkar1995using}, Ho-Kalman approach \cite{ho1966effective} and the subspace method raised in \cite{ljung1999system,van1995unifying,van2012subspace}, further modified as frequency domain subspace method in \cite{mckelvey1996subspace} when the inputs are single frequency (sine/cosine) signals. 
Recent works show that least-squares can be used to recover the Markov parameters and reconstruct $A,B,C,D$ from the Hankel matrix via the previously known Ho-Kalman algorithm \cite{ho1966effective, oymak2018non}. 
To identify a stable system from a single trajectory,  \cite{oymak2018non} estimates the impulse response and \cite{sarkar2019finite} estimates the Hankel matrix via least-squares. The latter provides optimal Hankel spectral norm error rates, however has suboptimal sample complexity (see the table in Section \ref{s:formulation}). While \cite{oymak2018non,sarkar2019finite} use random input,  \cite[Thm 1.1, 1.2]{tu2017non} use impulse and single frequency signal respectively as input. They both recover impulse response. These works assume known system order, or traverse the Hankel size $n$ to fit the system order. Ref. \cite{zheng2020non} proves that least-squares can identify any (including unstable) linear systems with multiple rollout data. Ref. \cite{reyhanian2021online} studies online system identification. It applies online gradient descent on least-squares loss and shows the identification error. Ref. \cite{fattahi2020learning} shows that, when the system is strictly stable ($\rho(A)<1$), the sample complexity is only polynomial in $(1-\rho(A))^{-1}$ and logarithmic in dimension.

There are several interesting generalizations of least squares with non-asymptotic guarantees for different goals. Refs. \cite{hazan2018spectral} and \cite{simchowitz2019learning} introduce filtering strategies on top of least squares. The filters in \cite{hazan2018spectral} is the top eigenvectors of a special deterministic matrix, used for output prediction in stable systems. Ref. \cite{simchowitz2019learning} uses filters in frequency domain to recover the system parameters of a stable system,   \cite{tsiamis2019finite} gives a non-asymptotic analysis for learning a Kalman filter system, which can also be applied to an auto-regressive setting. As an extension, \cite{dean2019safely} and \cite{mania2019certainty} apply system identification guarantee for robust control, where the system is identified and controlled in an episodic way. Ref. \cite{lu2021non} extended the online LQR to a non-episodic way. Ref. \cite{agarwal2019online} studies online control and regret analysis in adversarial setting, whose algorithm directly learns the policy in an end-to-end way. Ref. \cite{talebi2020online} controls an unknown unstable system with no initial stabilizing controller. Another area is system identification with non-linearity. Ref. \cite{mhammedi2020learning} learns a linear system using nonlinear output observations. Refs. \cite{oymak2019stochastic,khosravi2020nonlinear,foster2020learning,bahmani2019convex,sattar2020non} consider guarantees for certain nonlinear systems with state observations and \cite{mania2020active,wagenmaker2020active} study active learning where the new input adapts with respect to previous observations. Ref. \cite{rutledge2020finite} studies the estimation and proposes the subsequent model-based control algorithm with missing data. Refs. \cite{du2019mode,sattar2021identification} study clustering and identification for Markov jump system and Ref. \cite{du2021certainty} further analyzes the optimal control strategy based on the estimated system parameters.

Nuclear norm regularization has been shown to recover an unstructured low-rank matrix in a sample-efficient way in many settings (e.g.,   \cite{recht2010guaranteed,CandesMatrixComp1}). The regularized subspace method are introduced in \cite{hansson2012subspace,verhaegen2016n2sid}.
Refs. \cite{liu2013nuclear,fazel2013hankel} propose slightly different algorithms which recover low rank output Hankel matrix. Ref. \cite{grossmann2009system} specifies the application of Hankel nuclear norm regularization when some output data are missing. Ref. \cite{ayazoglu2012algorithm} proposes a fast algorithm on solving the regularization algorithm. All above regularization works emphasize on optimization algorithm implementation and have no statistical bounds. More recently \cite{cai2016robust} theoretically proves that a low order SISO system from multi-trajectory input-outputs can be recovered by this approach.
Ref. \cite{blomberg2016nuclear} gives a thorough analysis on Hankel nuclear norm regularization applied in system identification, including discussion on proper error metrics, role of rank/system order in formulating the problem, implementable algorithm and selection of tuning parameters.

The rest of the paper is organized as follows. Next section introduces the technical setup. Sections \ref{s:reg} proposes our results on nuclear norm regularization. Section \ref{sec:no_weight} discusses the role of the input distribution and establishes lower bounds.
Section \ref{s:ls} provides our results on least-squares estimator. 
Section \ref{s:model_selection} discusses model selection algorithms. Finally Section \ref{s:experiments} presents the numerical experiments\footnote{The code for experiments is in \url{https://github.com/sunyue93/sunyue93.github.io/blob/main/sysIdFiles.zip}.}.

\section{Problem Setup and Algorithms} \label{s:formulation}
Let $\|\cdot\|, \|\cdot\|_*, \|\cdot\|_F$ denote the spectral norm, nuclear norm and Frobenius norm respectively. Throughout, we estimate the first $2n-1$ terms of the impulse response denoted by $h$.  The system is excited by an input $u$ over the time interval $[0, t]$ 
and the output $y$ is measured at time $t$, i.e., 
\begin{align}\label{eq:conv}
    y_t = \sum_{i=1}^t h_{t+1-i}u_i + z_t.
\end{align}
We start by describing data acquisition models. Generally there are several rounds ($i$th round is denoted with super script $(i)$ in Fig. \ref{fig:data_acq}) of inputs sent into the system, and the output can be collected or neglected at arbitrary time. In the setting that we refer to as ``multi-rollout" (Fig. \ref{fig:data_acq}(b)), for each input signal $u^{(i)}$ we take only one output measurement $y_t$ at time $t=2n-1$ and then the system is restarted with a new input. Here the \emph{sample complexity} is $T$, the number of output measurements as well as the round of inputs. Recent papers (e.g., \cite{oymak2018non} and \cite{sarkar2019finite})  use the ``single rollout" model (Fig. \ref{fig:data_acq}(c)) where we apply an input signal from time $1$ to $T+2n-2$ without restart, and collect all output from time $2n-1$ to $T+2n-2$, in total $T$ output measurements; we use this model in the numerical experiments in Sec. \ref{s:experiments}.

\begin{figure*}[t!]
\centering{
{
{
\includegraphics[width = 0.3\textwidth]{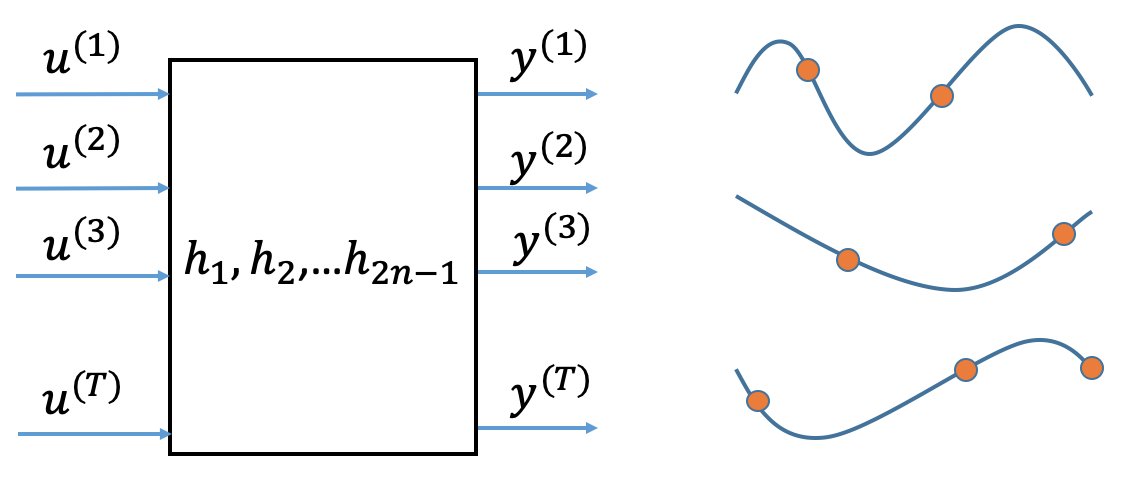}
}\hspace{-.1em}
{
\includegraphics[width = 0.3\textwidth]{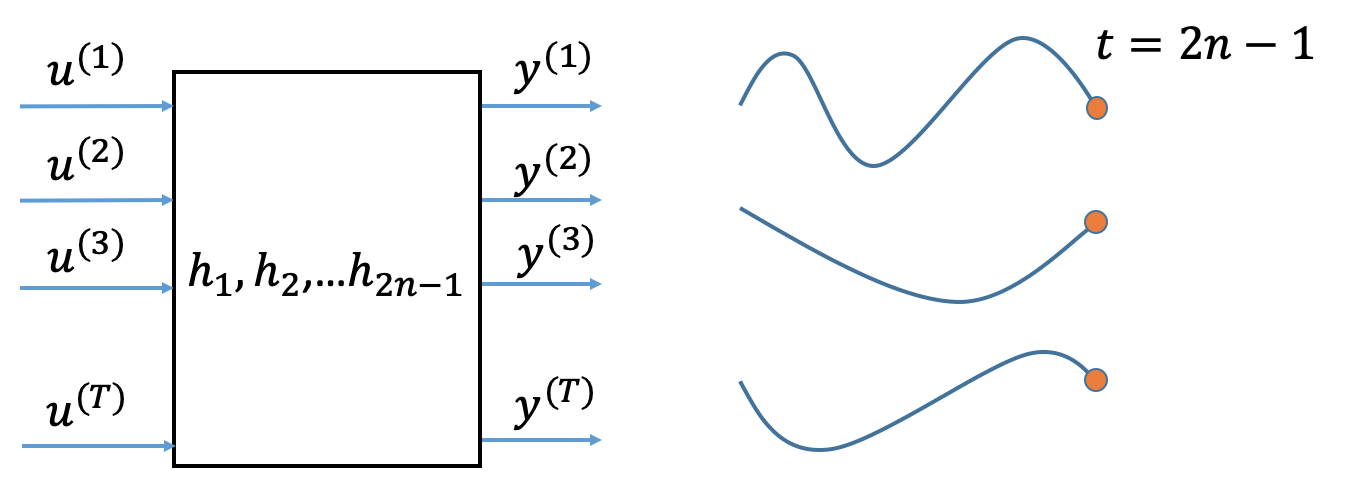}
}\hspace{-.1em}
{
\includegraphics[width = 0.3\textwidth]{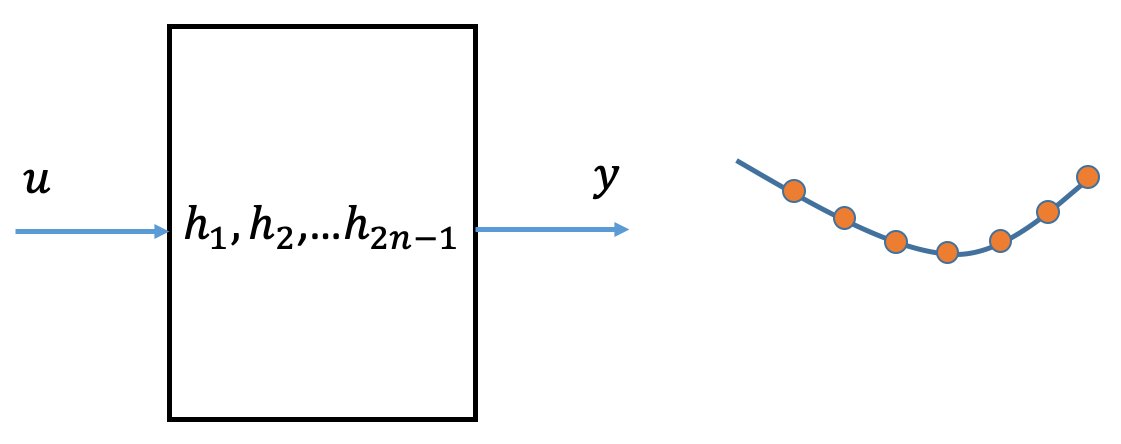}
}
}
}
{\caption{(a) Arbitrary sampling on output data, and two specific data aqcuisition models: (b) multi-rollout, and (c) single rollout.}\label{fig:data_acq}}
\end{figure*}
We consider two estimators in this paper: the  \emph{nuclear norm regularized estimator} and the \emph{least squares estimator} defined later.

We will bound the various error metrics mentioned earlier in terms of the sample complexity $T$, the true system order $R$, the dimension of impulse response $n\gg R$, and signal to noise ratio (SNR) defined as $\snr=\E[\|u\|_F^2/n]/\E[\|z\|_F^2]$. Table \ref{tab:1} provides a summary and comparison of these bounds.
All bounds are order-wise and hide constants and log factors. We can see that, with nuclear norm regularization, our paper matches the least squares impulse response and Hankel spectral error bound while sample complexity can be as small as ${\cal{O}}(R^2)$, and we can recover the impulse response with guaranteed suboptimal error when sample complexity is ${\cal{O}}(R)$. Our least square error bound matches the best error bounds among \cite{oymak2018non} and \cite{sarkar2019finite}, which is proven optimal for least squares. 

\newcommand{\tabincell}[2]{\begin{tabular}{@{}#1@{}}#2\end{tabular}} 

\begin{table*}[t!]
    \centering
    \caption{Comparison of recovery error of impulse response. The Hankel matrix is $n\times n$, the system order is $R$, and the number of samples is $T$, and $\sigma=1/\sqrt{\snr}$ denotes the noise level. LS-IR and LS-Hankel stands for least squares regression on the impulse response and on the Hankel matrix.}\label{tab:1}
    \begin{tabular}{|c|c|c|c|c|}
    \hline
       Paper & This work & This work & \cite{oymak2018non} & \cite{sarkar2019finite} \\\hline\hline
       Sample complexity & $R$ & $n$ & $n$ & $n^2$\\\hline
       Method & Nuc-norm & LS-IR & LS-IR & LS-Hankel\\ \hline
        \tabincell{c}{Impulse response error} & see \eqref{eq:reg_spec_error} & $\sigma\sqrt{{n}/{ T}}$ & $\sigma\sqrt{{n}/{ T}}$ & $(1+\sigma)\sqrt{{n}/{T}}$ \\\hline
       \tabincell{c}{Hankel spectral error} & see \eqref{eq:reg_spec_error} & $\sigma\sqrt{{n}/{ T}}$ & $\sigma {n}/{ \sqrt{T}}$ & $(1+\sigma)\sqrt{{n}/{T}}$\\\hline
    \end{tabular}
\end{table*}

Next, we discuss the design of the input signal and introduce \emph{input shaping matrix}.

\noindent{\bf{Input shaping:}} Note that the $\mathcal{H}$ operator does not preserve the Euclidean norm, so \cite{cai2016robust} proposes using a normalized operator $\cG$, where they first define the weights

\begin{align}
    K_j = \left\{
    \begin{array}{rcl} \sqrt{j},& 1\le  j\le n, \\ \sqrt{2n-j}, & n< j\le 2n-1.
    \end{array} \right.\label{eq K-shape}
\end{align}
and let $K\in\R^{(2n-1)p\times (2n-1)p}$ be a block diagonal matrix where the $j$th diagonal block of size $p\times p$ is equal to $K_jI_{p\times p}$. In other words,
\begin{align*}
    K = \begin{bmatrix}
    K_1I & 0 & 0 & ... & 0\\
    0 & K_2I & 0 & ... & 0\\
    ...\\
    0 & 0 & 0 & ... & K_{2n-1}I
    \end{bmatrix}
\end{align*}
Define the mapping $\cG(h) = \cH(K^{-1}h)$. 
In other words, if $\beta = Kh$ then $\cG(\beta) = \cH(h)$. Define $\cG^*:\R^{mn\times np}\rightarrow\R^{m\times (2n-1)p}$ as the adjoint of $\cG$, where $[\cG^*(M)]_i = \sum_{j+k-1=i}M_{(j)(k)} / K_i$ if we denote the $j,k$-th block of $M$ (defined in \eqref{eq:H}) by $M_{(j)(k)}$. Using this change of variable and letting $\Ub = \bar \Ub K^{-1}$, problem \eqref{eq:lasso_prob_h} can be written as   
\begin{align}\label{eq:lasso_prob}
    \hat{\beta}=\arg\min_{\beta'}&\quad  \frac{1}{2}\|\Ub \beta' - y\|_F^2 + \lambda\|\mathcal{G}(\beta')\|_*.
\end{align}

\section{Hankel Nuclear Norm Regularization}\label{s:reg}
To promote a low-rank Hankel matrix, we add nuclear norm regularization to the quadratic-loss objective and solve the regularized regression problem. 
Here we give a finite sample analysis for the recovery of the Hankel matrix and the impulse response found via this approach. We consider a random input matrix $\bar{\Ub}$ and observe the corresponding noisy output vector $y$ as in \eqref{eq:conv}. 

\subsection{Statistical guarantees for Hankel-regularization}
The theorem below shows that Hankel-regularization achieves near-optimal sample complexity with similarly strong estimation error rates that decay as $1/\sqrt{T}$.
\begin{theorem}\label{main_gauss}
    Consider the problem \eqref{eq:lasso_prob_h} in the MISO (multi-input single-output) setting ($m=1$, $p$ input channels). Suppose the system order is $R$, $\bar\Ub\in\mathbb{R}^{T\times (2n-1)p}$, each row consists of an input rollout $u^{(i)}\in\R^{(2n-1)p}$, and the scaled input $\Ub = \bar{\Ub}K^{-1}$ has i.i.d Gaussian entries. Let $\snr=\E[\|u\|^2/n]/\E[\|z\|^2]$ and $\sigma=1/\sqrt{\snr}$. For some constant $C$, let $\lambda = C\sigma\sqrt{\frac{pn}{T}}\log(n)$. Then, with probability $1-{\cal{O}}(R\log(n) \sqrt{p/ T})$, solving \eqref{eq:lasso_prob_h} returns $\hat h$ such that\footnote{In this theorem, we make the dependence on $\log(n)$ explicit.}
    \begin{align} 
     \quad \frac{\|\hat{h} - h\|_2}{\sqrt{2}} &\leq \|\cH(\hat{h} - h)\|\lesssim   
     \begin{cases} \sqrt{\frac{np}{\snr\times T}}\log(n) \quad\text{if}\quad T\gtrsim \bar{R}\\
     \sqrt{\frac{Rnp}{\snr\times T}}\log(n)\quad\text{if}\quad  R\lesssim T\lesssim \bar{R}\end{cases}\label{eq:reg_spec_error}
    \end{align}
    where $\bar{R}=\min(R^2\log^2(n),n)$.
\end{theorem}
Theorem \ref{main_gauss} jointly bounds the impulse response and Hankel spectral errors of the system under mild conditions. We highlight the improvements that our bounds provide:
When the system is low order, the sample complexity $T$ is logarithmic in $n$ and improves upon the ${\cal{O}}(n)$ bound of the least-squares algorithm. The number of samples for recovering an order-$R$ system is ${\cal{O}}(R\log(n))$, where the SISO case is also proven in \cite{cai2016robust}, and we follow a similar proof. 
The error rate with respect to the system parameters $n,R,T$ is same as \cite{oymak2018non,sarkar2019finite,tu2017non} (e.g. compare to Thm. \ref{thm ls spectral}). We can choose $\lambda$ from a range of a constant multiplicative factors, like $\lambda = [1,2]\cdot C\sigma\sqrt{\frac{pn}{T}}\log(n)$. When $\lambda$ is in this range, there is always a sample complexity $T$ such that \eqref{eq:reg_spec_error} holds. The flexibility of $\lambda$ makes tuning $\lambda$ for Algorithm \ref{algo:1} easier.

The regularized method also has the intrinsic advantage that it does not require knowledge of the rank or the singular values of the Hankel matrix beforehand. 
Tuning method Algorithm \ref{algo:1} and numerical experiments on real data in Sec. \ref{s:experiments} demonstrate the performance and robustness of the regularized method.

\subsection{Sample complexity lower bounds for IID inputs and the Importance of Input Shape}\label{sec:no_weight}

Theorem \ref{main_gauss} uses shaped inputs whereas, in practice, one might expect that i.i.d.~input sequence (without shaping matrix $K$) should be sufficient for recovering the impulse response with near-optimal sample size. For instance, \cite{oymak2018non} proves an optimal sample complexity bound for system identification via least-squares and i.i.d.~standard normal inputs. Naturally, we ask: Does Hankel-regularization enjoy similar performance guarantees with i.i.d.~inputs? Do we really need input design?

The following theorem proves that for a special system with order $r=1$, the sample complexity of the problem under i.i.d.~input is no less than $n^{1/3}$, compared to $O(\log n)$ under the shaped input setting. This is accomplished by carefully lower bounding the Gaussian width induced by the Hankel-regularization. Gaussian width directly corresponds to the square-root of the sample complexity of the problem required for recovery in high probability \cite[Thm. 1]{mccoy2013achievable}. Thus, Theorem \ref{thm:counter_iid} shows that the sample complexity with i.i.d.~input can indeed be provably larger than with shaped input.

\begin{theorem}\label{thm:counter_iid}
Suppose the impulse response $h$ of the system is $h_t=1,\ \forall t\ge 1$, which is order $1$. Consider the tangent cone associated with Hankel-regularization (normalized to unit $\ell_2$-norm) defined as $\{x/\|x\|\ |\ \|\cH(h+x)\|_* \le \|\cH(h)\|_*\}$. The Gaussian width of this set is lower bounded by $Cn^{1/6}$ for some constant $C>0$.
\end{theorem}
This implies that, in the noiseless setting, the sample complexity to recover the impulse response is $T \gtrsim n^{1/3}$, which is larger than $\log n$ dependence with shaped input. This result is rather counter-intuitive since i.i.d.~inputs are often optimal for structured parameter estimation tasks (e.g.~compressed sensing, low-rank matrix estimation). In contrast, our result shows the provable benefit of input shaping. 

\section{Refined Bounds for the Least-Squares Estimator}\label{s:ls}

To better understand the Hankel-regularization bound of Theorem \ref{main_gauss}, one can contrast it with the performance of unregularized least-squares. Interestingly, the existing least-squares bounds are not tight when it comes to Hankel spectral norm. In this section, we revisit the least-squares estimator, tighten existing bounds, and contrast the result with our Theorem \ref{main_gauss}. We consider the MIMO setup where $y\in\R^{T\times m}$ and $h\in\R^{(2n-1)p\times m}$. This is obtained by setting $\lambda=0$ in \eqref{eq:lasso_prob_h}, hence the estimator is given via the pseudo-inverse
\begin{align}\label{eq:generic_ls}
    \hat{h}:=h + \bar{\Ub}^\dag z=\min_{h'}~  \frac{1}{2}\|\bar{\Ub}h' - y\|_F^2.
\end{align}

The next theorem bounds the error when inputs and noise are randomly generated.
\begin{theorem}\label{thm ls spectral}
    Denote the solution to \eqref{eq:generic_ls} as $\hat h$. Let $\bar{\Ub}\in\R^{T\times (2n-1)p}$ be input matrix obtained from multiple rollouts, with i.i.d.~standard normal entries, $y\in\R^{T\times m}$ be the corresponding outputs and $z\in\R^{T\times m}$ be~the noise matrix with i.i.d. $\Nn(0,\sigma_z^2)$ entries. Then the spectral norm error obeys
    $\|\mathcal{H}(\hat h - h)\| \lesssim \sigma_ z \sqrt{\frac{mnp}{T}}\log (np)$ when $T\gtrsim mnp$.
\end{theorem}

This theorem improves the spectral norm bound compared to \cite{oymak2018non}, which naively bounds the spectral norm in terms of IR error using the right-hand side of \eqref{spec relate}. Instead, we show that spectral error is same as the IR error up to a log factor (when there is only output noise).  
We remark that $O(\sigma_{z}\sqrt{np/T})$ is a tight lower bound for $\|\Hc(h-\hat{h})\|$ as well as $\|h - \hat h\|$ \cite{oymak2018non,arias2012fundamental}. The proof of the theorem above is in Appendix \ref{s:ls_append}. Note that, we apply the i.i.d input here for least squares which is different from regularized algorithm. It works with at least $\order{n}$ samples, whereas the sample complexity for the regularized algorithm in Theorem \ref{main_gauss} is $\order{R}$.

\section{Model Selection for Hankel-Regularized System ID}\label{s:model_selection}
In Thm. \ref{main_gauss}, we established the recovery error for system's impulse response for a particular parameter choice $\lambda$, which depends on the noise level $\sigma$. In practice, we do not know the noise level and the optimal $\lambda$ choice is data-dependent. Thus, given a set of parameter candidates $\Lambda\subset\R^+$, one can evaluate $\lambda\in\Lambda$ and minimize the validation error to perform model selection. Algorithm \ref{algo:1} summarizes our training and validation procedure where $|\Lambda|$ denotes the cardinality of $\Lambda$. The theorem below states the performance guarantee for this algorithm. 

\begin{theorem}\label{thm:e2e} Consider the setting of Theorem~\ref{main_gauss}. Sample $T$ i.i.d.~training rollouts $(\Ub,y)$ and $T_{\text{val}}$ i.i.d.~validation rollouts $(\Ub_{\text{val}},y_{\text{val}})$. Set $\lambda^* = C\sigma\sqrt{\frac{pn}{T}}\log(n)$ which is the choice in Thm.~\ref{main_gauss}. Fix failure probability $P\in(0,1)$. Suppose that:

\textbf{(a)} There is a candidate $\hat \lambda\in\Lambda$ obeying $\lambda^*/2 \le \hat\lambda \le 2\lambda^*$.

\textbf{(b)} Validation set obeys $T_{\text{val}} \gtrsim \left( \frac{T\log^2(|\Lambda|/P)}{R\log^2(n)}\right)^{1/3}$.

Set $\bar{R}=\min(R^2,n)$. With probability at least $1-P$, Algorithm \ref{algo:1}
achieves an estimation error equivalent to \eqref{eq:reg_spec_error}:
\begin{align} 
  \|\cH(\hat{h} - h)\| \lesssim   
 \begin{cases} \sqrt{\frac{np}{\snr\times T}}\log(n),~\text{if}\ T\gtrsim \bar{R},\\
 \sqrt{\frac{Rnp}{\snr\times T}}\log(n),~\text{if}\ R\lesssim T\lesssim \bar{R}.\end{cases}
 \label{eq:reg_spec_error_model_select}
\end{align}
\end{theorem}

\begin{algorithm}[tp!]
\caption{Joint System Identification \& Model Selection}\label{algo:1}
\begin{algorithmic}
 \REQUIRE{\emph{\red{$T$ training rollouts:}}~~~Input features $\bar \Ub$, outputs $y$\\
 \emph{\red{$T_{\text{val}}$ validation rollouts:}}~~~Input features $\bar \Ub_{\mathrm{val}}$, outputs $y_{\mathrm{val}}$
 \hspace{5pt}\emph{\red{Hyperparameters:}} Hankel dimension $n$, candidate set $\Lambda$.}
\STATE{\texttt{\green{\% Train a model for each $\lambda\in\Lambda$}}}
\FOR{$\lambda \in \Lambda$}
\STATE{\emph{\red{Training phase:}} Solve \eqref{eq:lasso_prob_h}. Record the solution $\hat{h}_{\lambda}$.
}\ENDFOR
\STATE{\texttt{\green{\% Return the minimum validation error}}}
\STATE{\emph{\red{Model selection:}} $\hat h=\underset{\hat{h}_\lambda:\lambda \in\Lambda}{\arg\min}\|\bar \Ub_{\mathrm{val}} \hat h_{\lambda} - y_{\mathrm{val}}\|_2^2$.
}
\RETURN{$\hat h$}
\end{algorithmic}
\end{algorithm}

In a nutshell, this result shows that as long as candidate set contains a reasonable hyperparameter (Condition \textbf{(a)}), using few validation data (Condition \textbf{(b)}), one can compete with the hindsight parameter choice of Theorem~\ref{main_gauss}. $T_{\text{val}}$ mildly depends on $T$ and only scales poly-logarithmically in $|\Lambda|$ which is consistent with the model selection literature \cite{arlot2010survey}.

In Algorithm \ref{algo:1}, we fix the size of the Hankel matrix (which is usually large/overparameterized) and tune $\lambda>0$. In contrast to this, for OLS and the least-squares procedure of \cite{sarkar2019finite}, model selection is accomplished by varying the size of the Hankel matrix. The next section provides experiments and contrasts these methods and provides insights on how regularization can improve over least-squares for certain class of dynamical systems.

In Appendix \ref{sec: e2e hankel}, we further compare the model selection methods of Hankel-regularization and least-squares, and argue that, the proposed Algorithm \ref{algo:1} requires less data than tuning the least squares method.

\section{Experiments and Insights}\label{s:experiments}

\subsection{Experiments with Synthetic Data}\label{s:experiments slow decay}

 \textbf{When does regularization beat least-squares? Low-order slow-decay systems (Fig. \ref{fig:heavy}).} We use an experiment with synthetic data to answer this question. So far, we showed that for fixed Hankel size, nuclear norm regularization requires less data than unregularized least-squares especially when the Hankel size is set to be large (Table \ref{tab:1}). However, for least squares, one can choose to use a \emph{smaller Hankel size} with $n\approx R$, so that we solve a problem of small dimension compared to $n\gg R$. We ask if there is a scenario in which fine-tuned nuclear norm regularization strictly outperforms fine-tuned least-squares. 

In what follows, we discuss a single trajectory scenario. 
An advantage of the Hankel-regularization is that, it addresses scenarios which can benefit from large $n$ while both the sample complexity $T$ and the system order $R$ are small. Given choice of $n$, in a single trajectory setting, least-squares suffers an error of order $\rho(A)^{2n}(1-\rho(A)^n)^{-1}$ (Thm 3.1 of \cite{oymak2018non}). 
This error arises from FIR truncation of impulse response (to $n$ terms) and occurs for both regularized and unregularized algorithms. In essence, due to FIR truncation, the problem effectively incurs an output noise of order $\rho(A)^{2n}(1-\rho(A)^n)^{-1}$. Thus, if the system decays slowly, i.e., $\rho(A)\approx 1$, we will suffer from significant truncation error.

\begin{figure*}[t!]
\centering{
{
\includegraphics[width=0.3\textwidth,height=0.21\textwidth]{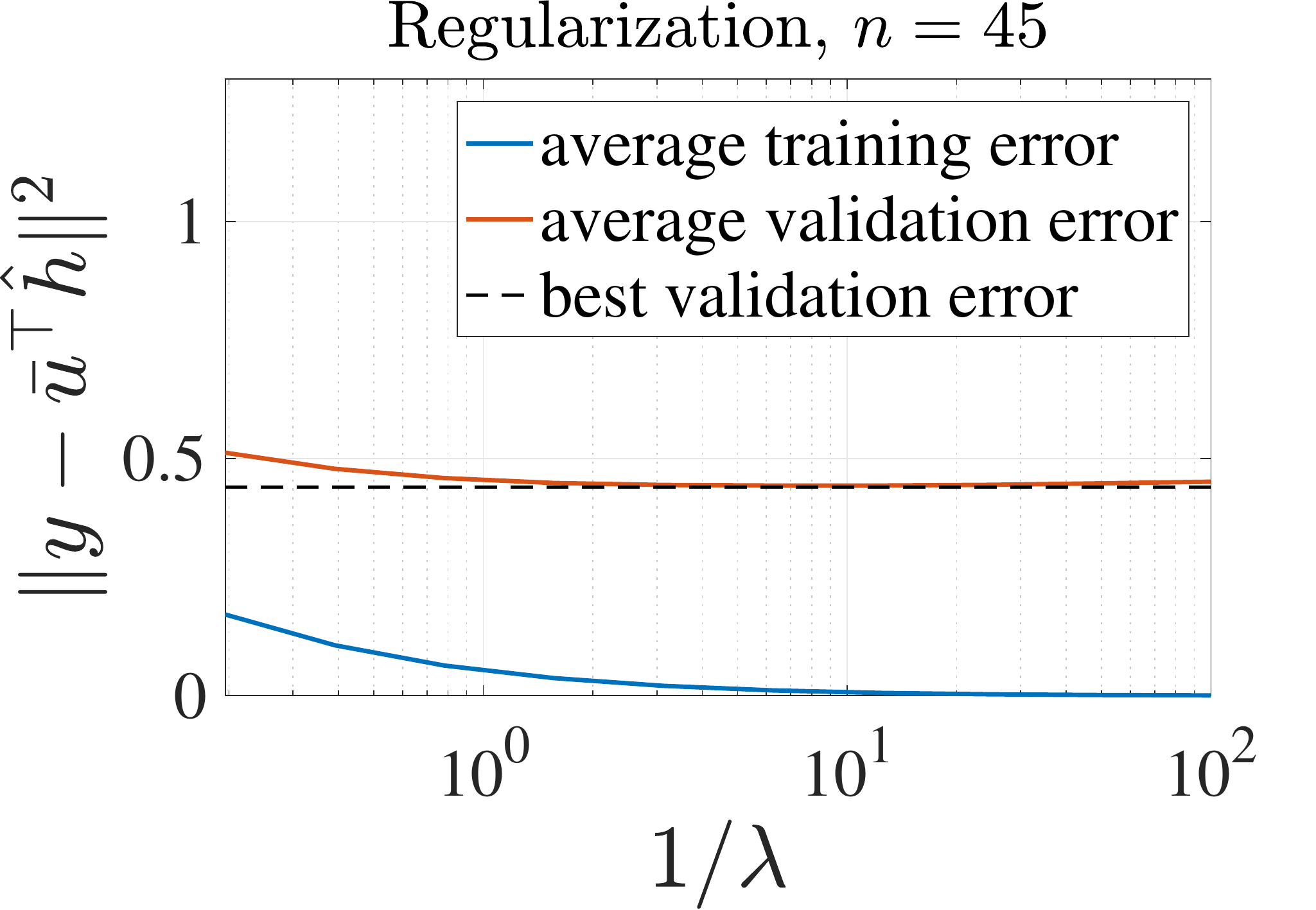}
}
{
\includegraphics[width=0.3\textwidth,height=0.21\textwidth]{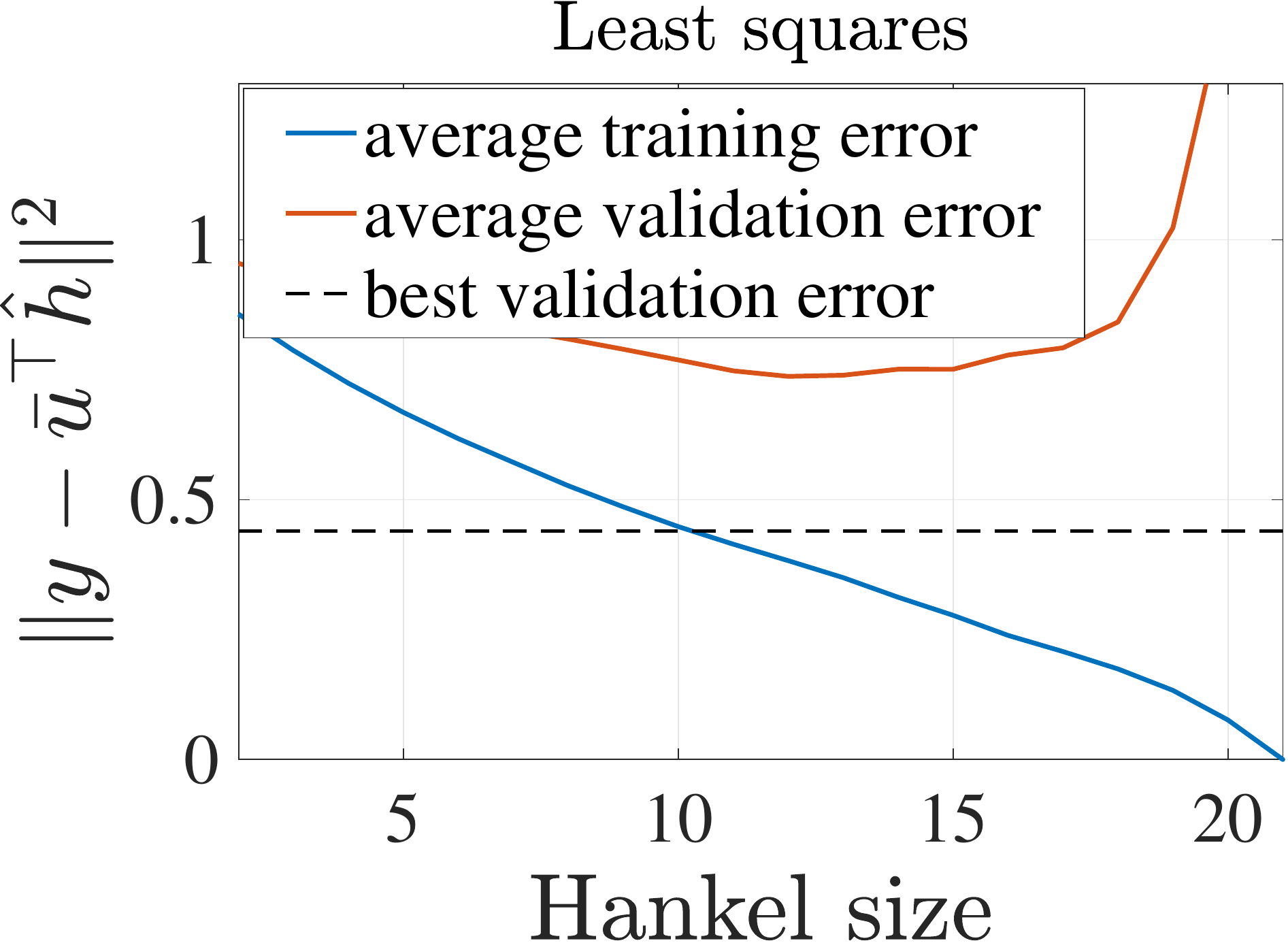}
}
{
\includegraphics[width=0.3\textwidth,height=0.21\textwidth]{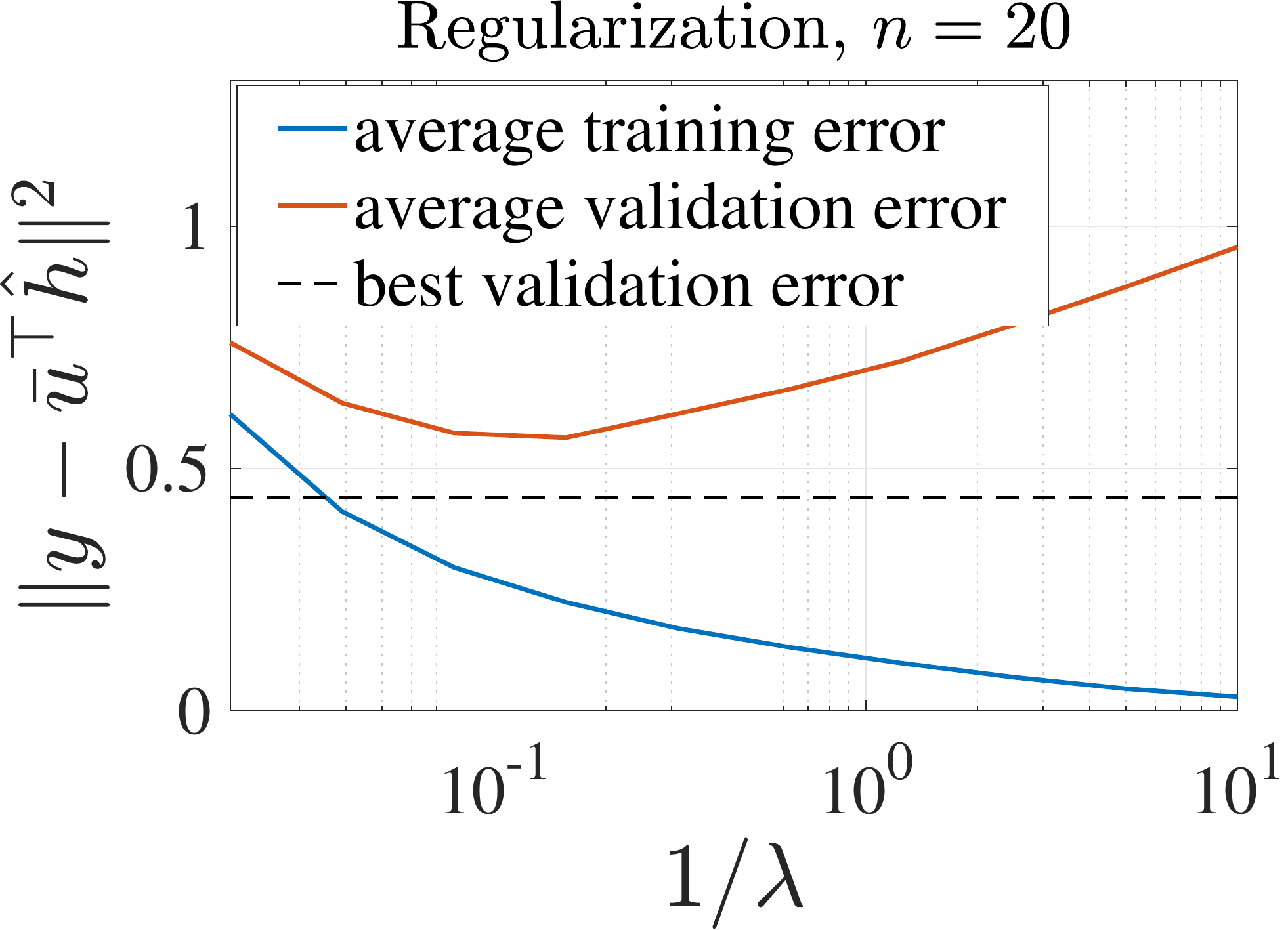}
}
}
{\caption{Synthetic data, single rollout. Order-1 system with pole$=0.98$. Recovery by \emph{Hankel-regularized} algorithm varying $\lambda$ and \emph{least squares} varying Hankel size $n$. Training sample size is $40$ and validation sample size is $800$. The figures are the training/validation error with (1) Hankel-regularized algorithm, varying $\lambda$ with $n = 45$; and (2) least-squares with varying $n$; (3) Hankel-regularized algorithm, varying $\lambda$ with $n = 20$ (small Hankel size). Black line is the best validation error achieved in the first figure.
} \label{fig:heavy}}
\end{figure*}
 
As an example, suppose sample size is $T=40$ and $\rho(A) = 0.98$. If the problem is kept over-determined (i.e. $n< 40$), then $n$ will not be large enough to make the truncation error $\rho(A)^{2n}(1-\rho(A)^n)^{-1}\approx 0.36$ vanishingly small. In contrast, Hankel-regularization can intuitively recover such slowly-decaying systems by choosing a large Hankel size $n$ as its sample complexity is (mostly) independent of $n$ (as in the setting of Theorem \ref{main_gauss}). This motivates us to compare the performance on recovering systems with \emph{low-order slow-decay}. In Fig. \ref{fig:heavy}, we set up an order-$1$ system with a pole at $0.98$ and generate single rollout data with size $40$. We tune $\lambda$ when applying regularized algorithm with $n=45$ (as it is safe to choose a large Hankel dimension), whereas in unregularized method, Hankel size cannot be larger than $20$ ($n\times n$ Hankel has $2n-1$ parameters and we need least squares to remain overdetermined). With these in mind, in the first two figures we can see that the best validation error of regularization algorithm is $0.44$, which is significantly smaller than the unregularized validation error $0.73$. In the third figure, we use regularized least squares with $n=20$, which also causes large truncation error (due to large $\rho(A)$) compared to the initial choice of $n=45$ (the first figure). In this case, the best validation error is $0.56$ which is again noticeably worse than the $0.44$ error in the first figure. 

When the number of variables $2n-1$ is larger than $T$, the system identification problem is overparameterized and there can be infinitely many impulse responses that achieve zero squared loss on the training dataset. This happens in the first figure when $\lambda\rightarrow 0$, and in the second figure when $n$ is large.  In this case, regularized algorithm chooses the solution with the smallest Hankel nuclear norm and the least squares chooses the one with smallest $\ell_2$ norm.   The first figure has smaller validation error when $1/\lambda$ tends to infinity. We verified that, when regularization weight is $0$, the setup in the first figure achieves validation error $0.58$, which is better than least squares for $n=45$ with validation error $1.60$. So among the solutions that overfits the training dataset, the one with small Hankel nuclear norm has better generalization performance when the true system is low order. 

\subsection{Experiments with DaISy Dataset}

Our experiment uses the DaISy dataset \cite{de1997daisy}, where a known input signal (not random) is applied and the resulting noisy output trajectory is measured. Using the input and output matrices
\begin{align}\label{eq:single_input_rip}
    \Ub &= \begin{bmatrix}
    u_{2n-1}^T & u_{2n-2}^T & ... & u_1^T\\
    u_{2n}^T & u_{2n-1}^T & ... & u_2^T\\
    ...\\
    u_{2n+T-2}^T&  u_{2n+T-3}^T & ... & u_{T}^T
    \end{bmatrix},\\
    y &= [y_{2n-1},...,y_{2n + T-2}],
\end{align}
we solve the optimization problem \eqref{eq:lasso_prob_h} using single trajectory data.

While the input model is single instead of multiple rollout, experiments will demonstrate the advantage of Hankel-regularization over least-squares in terms of sample complexity, singular value gap and ease of tuning.

\begin{figure*}[t!]
\centering{
{
\includegraphics[width=0.3\textwidth,height=0.21\textwidth]{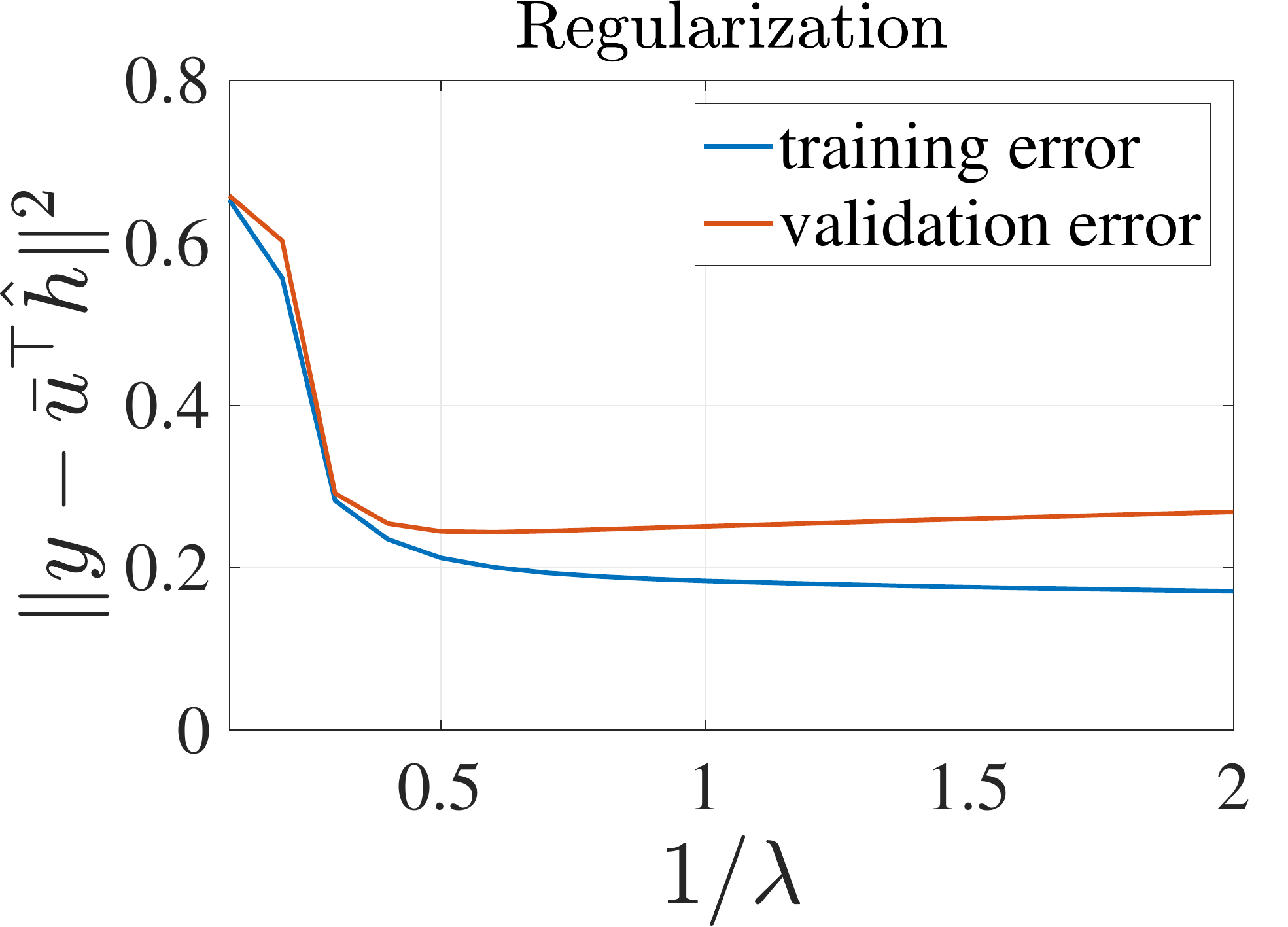}
}\hspace{-0.5em}
{
\includegraphics[width=0.3\textwidth,height=0.21\textwidth]{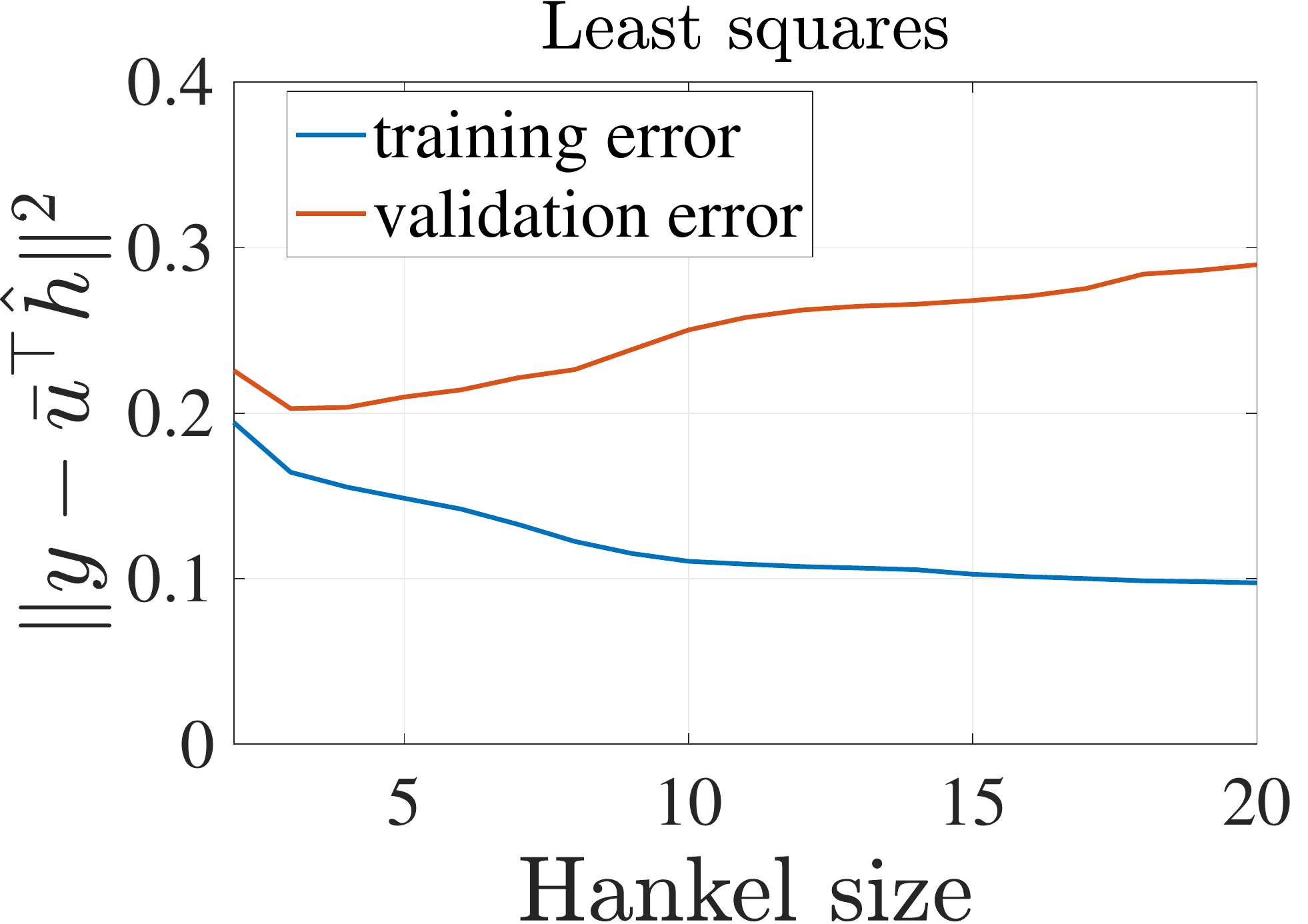}
}\hspace{-0.5em}
{
\includegraphics[width=0.3\textwidth,height=0.21\textwidth]{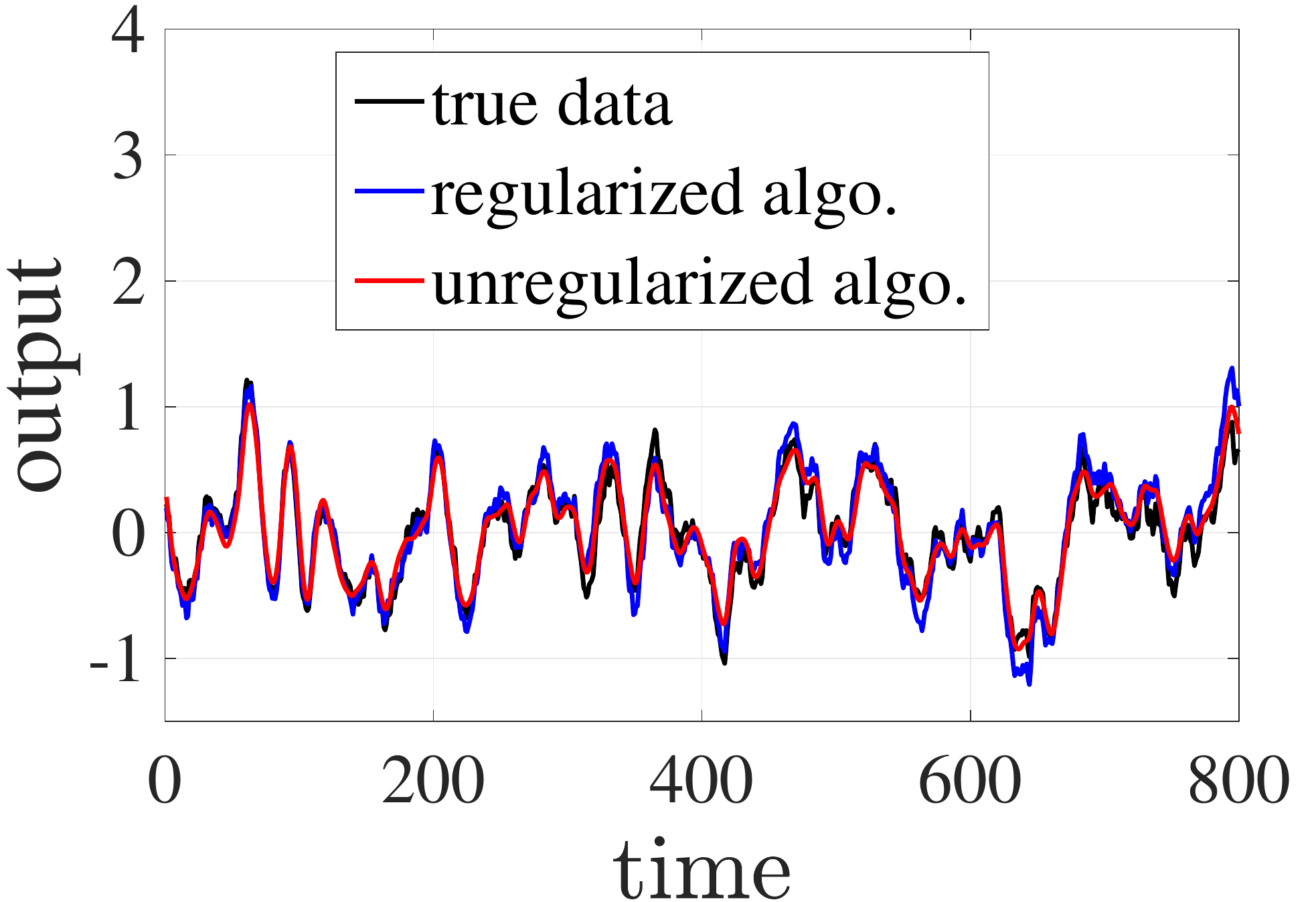}
}
}
{\caption{System identification for CD player arm data. Training data size = 200 and validation data size = 600. The first two figures are the training/validation errors of varying $\lambda$ in Hankel regularized algorithm ($n=10$), and training/validation errors of varying Hankel size $n$ in least squares. The last figure is the output trajectory of the true system and the recovered systems (best validation chosen for each).} \label{fig:error_bigr_CD}}
\end{figure*}

\begin{figure}[t!]
\centering{
{
\includegraphics[width=0.23\textwidth]{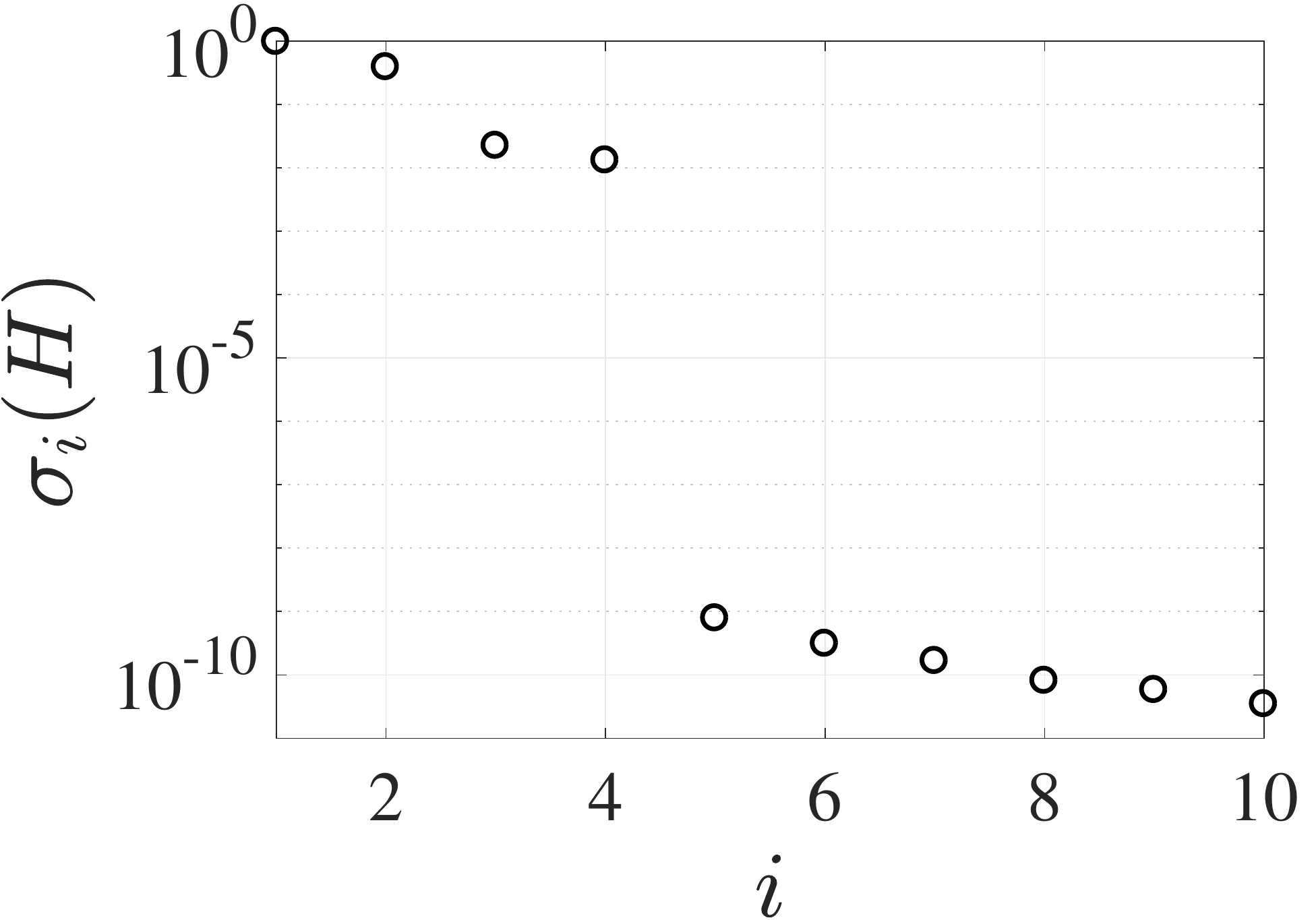}
}\hspace{-0.5em}
{
\includegraphics[width=0.23\textwidth]{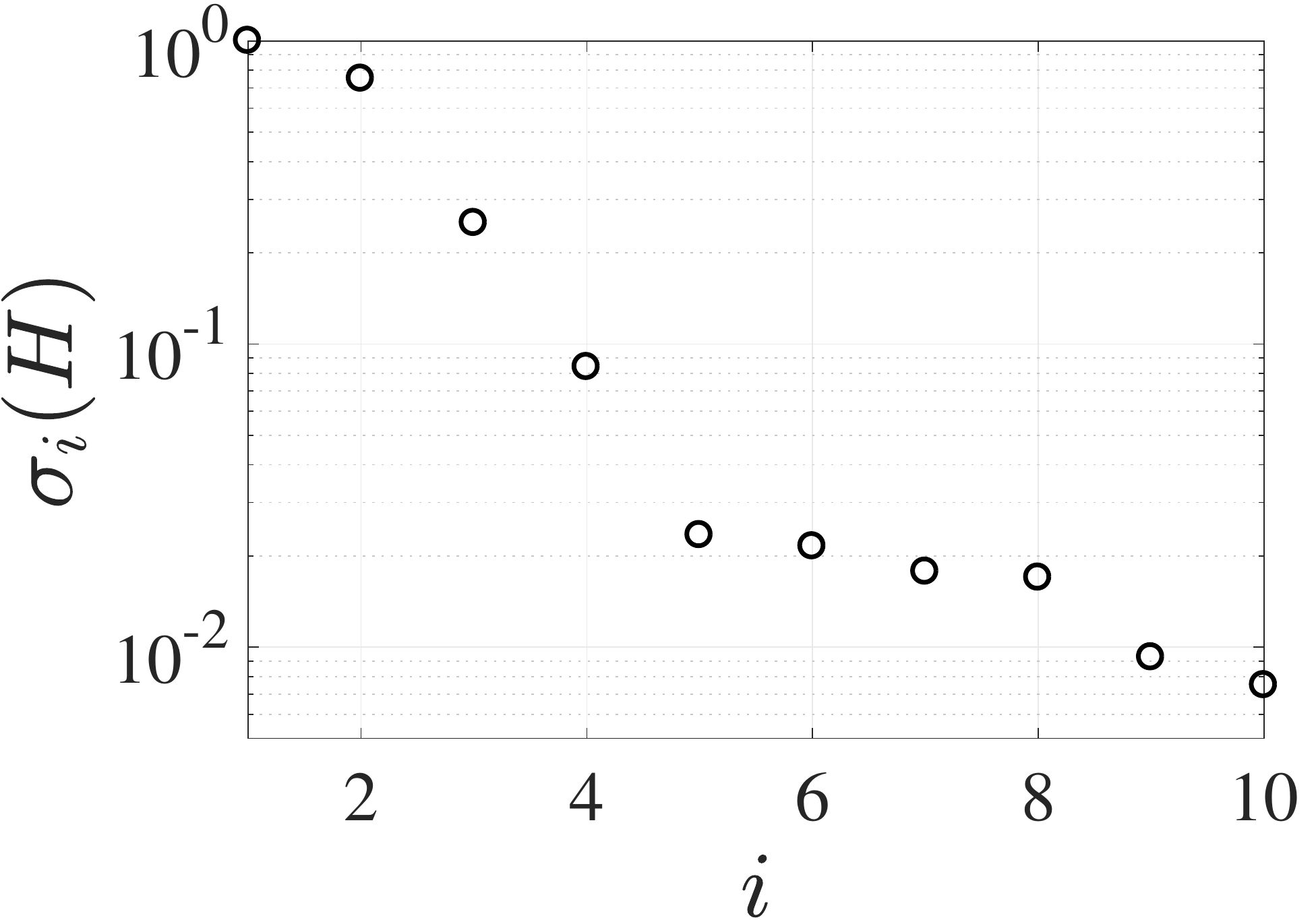}
}
{
\includegraphics[width=0.23\textwidth]{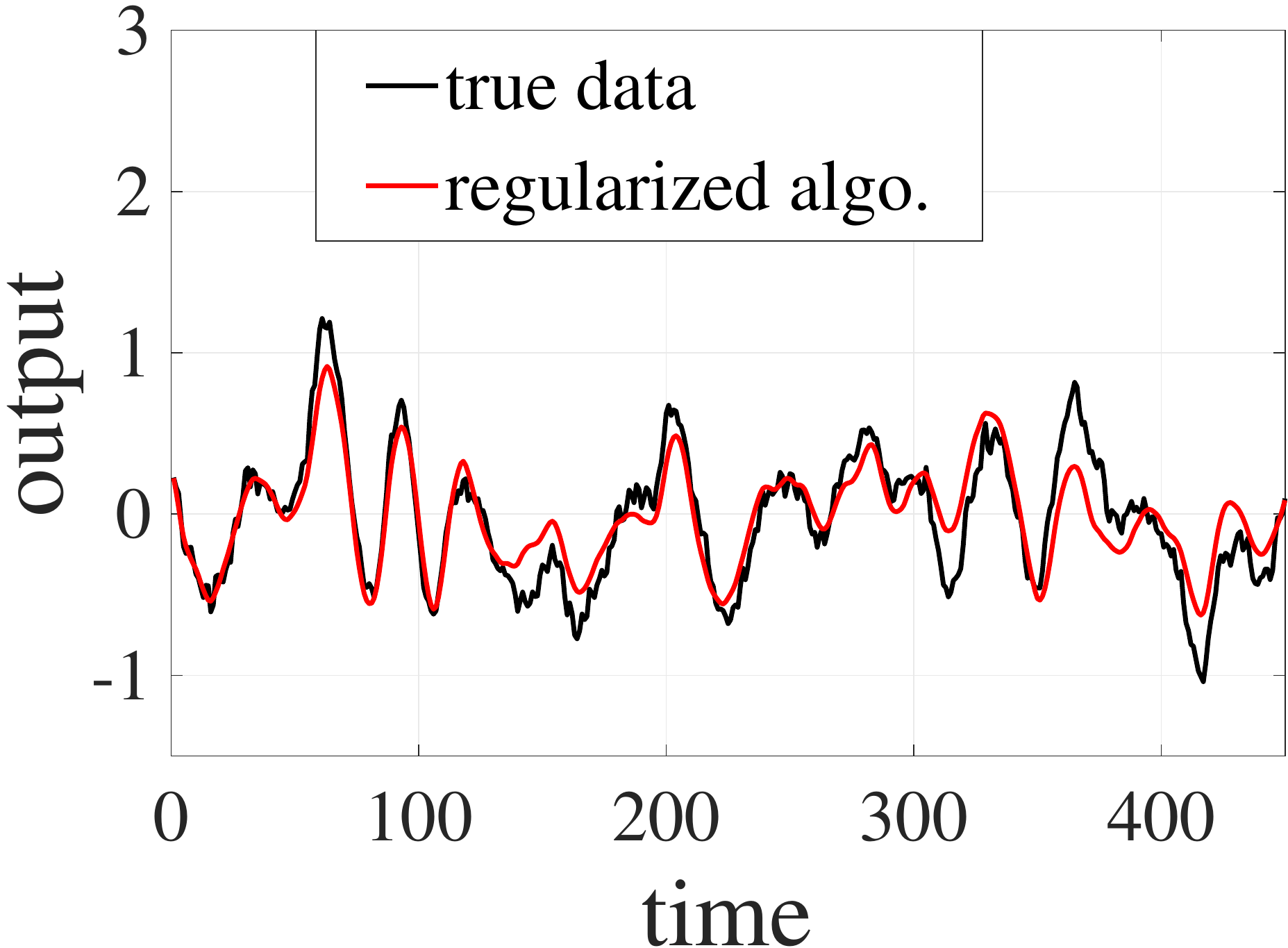}
}\hspace{-0.5em}
{
\includegraphics[width=0.23\textwidth]{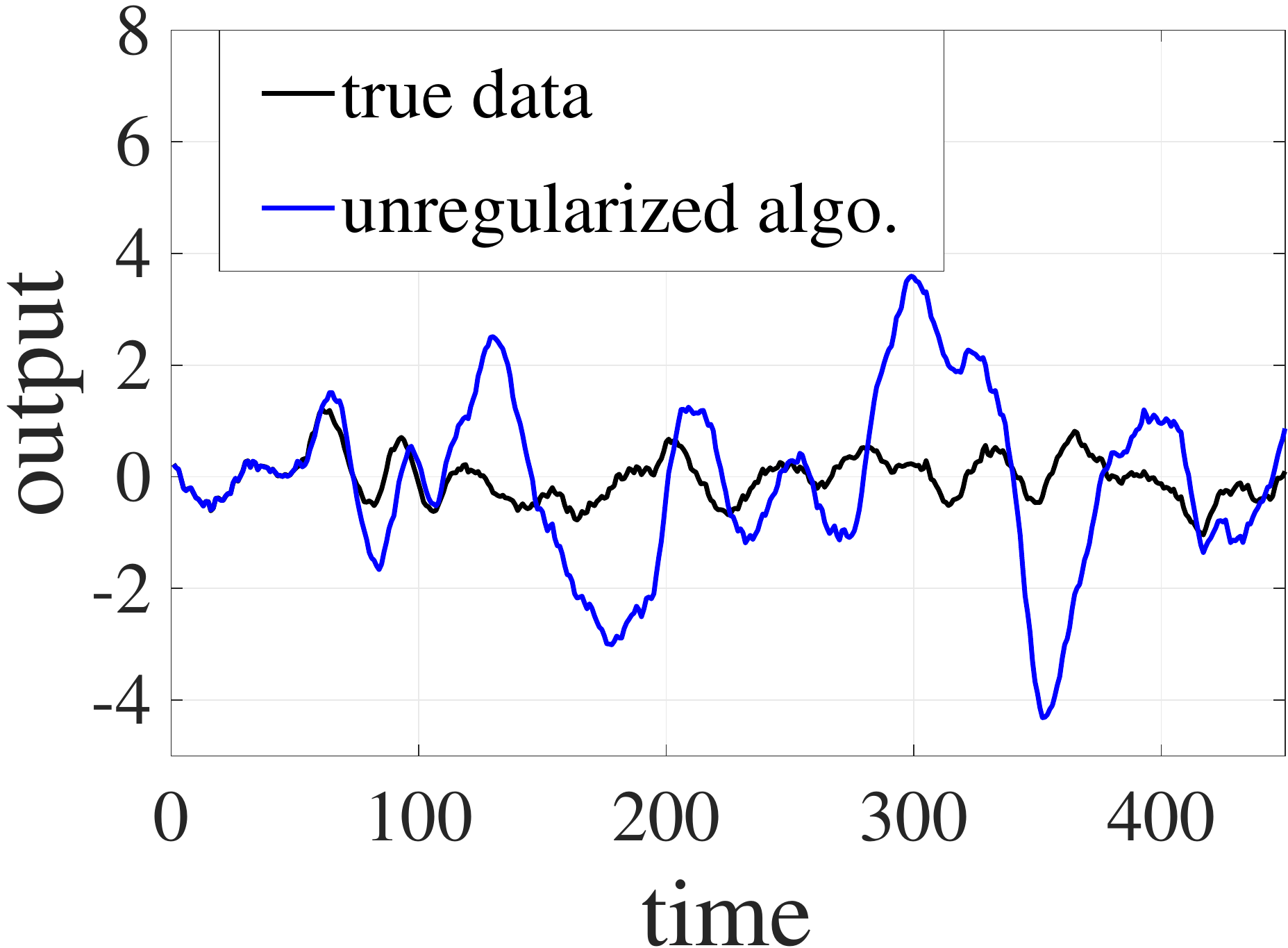}
}
}
{\caption{The first two: CD player arm data, singular values of the Hankel matrix with Hankel regularized algorithm and least squares. The last two: Recovery by Hankel regularized algorithm and least squares when Hankel matrix is $10\times10$. Training data size is $50$ and validation data size is $400$.} \label{fig:bigr_CD_eig}}
\end{figure}

\begin{figure}[t!]
\centering{
{
\includegraphics[width=0.23\textwidth]{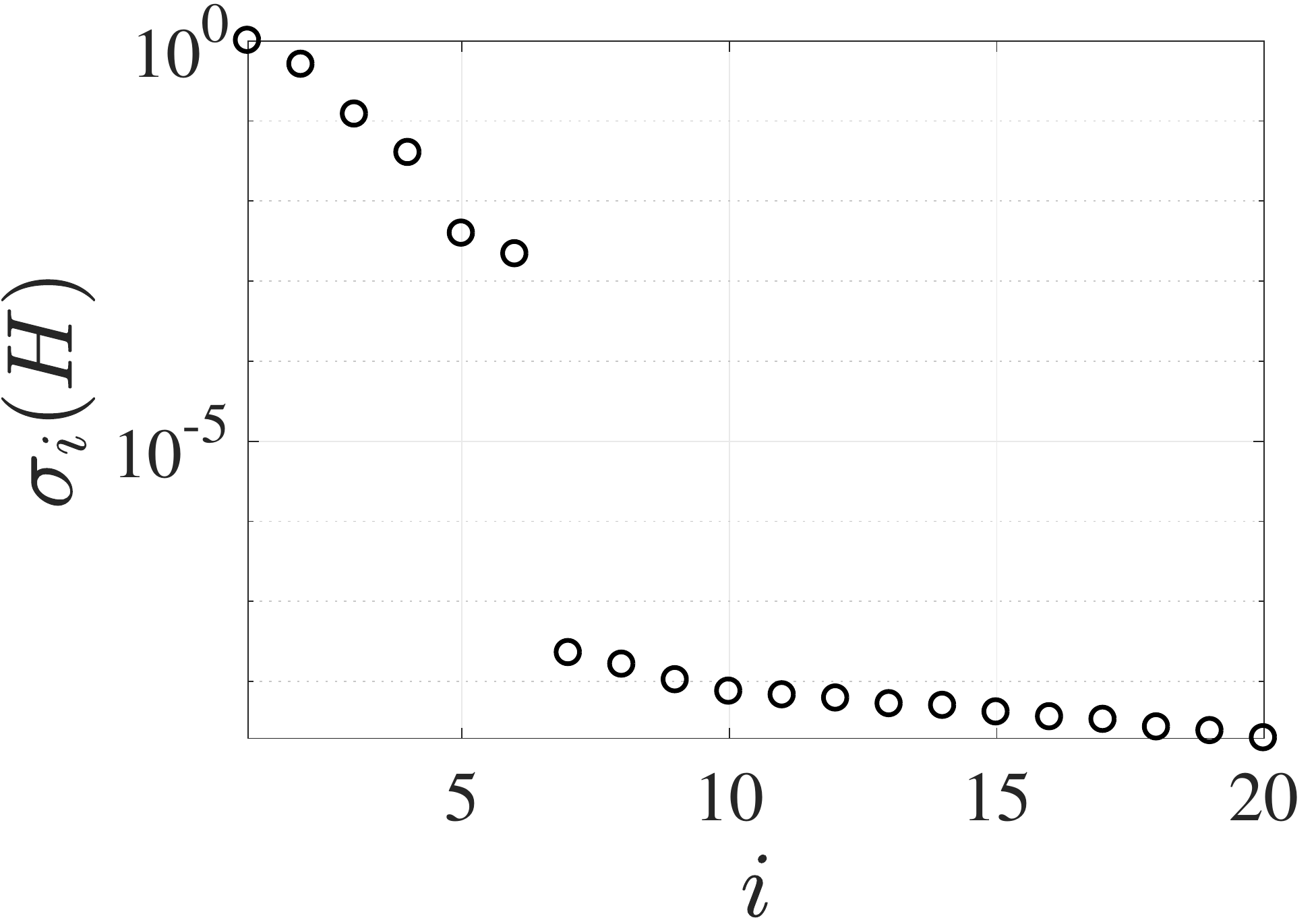}
}\hspace{-0.5em}
{
\includegraphics[width=0.23\textwidth]{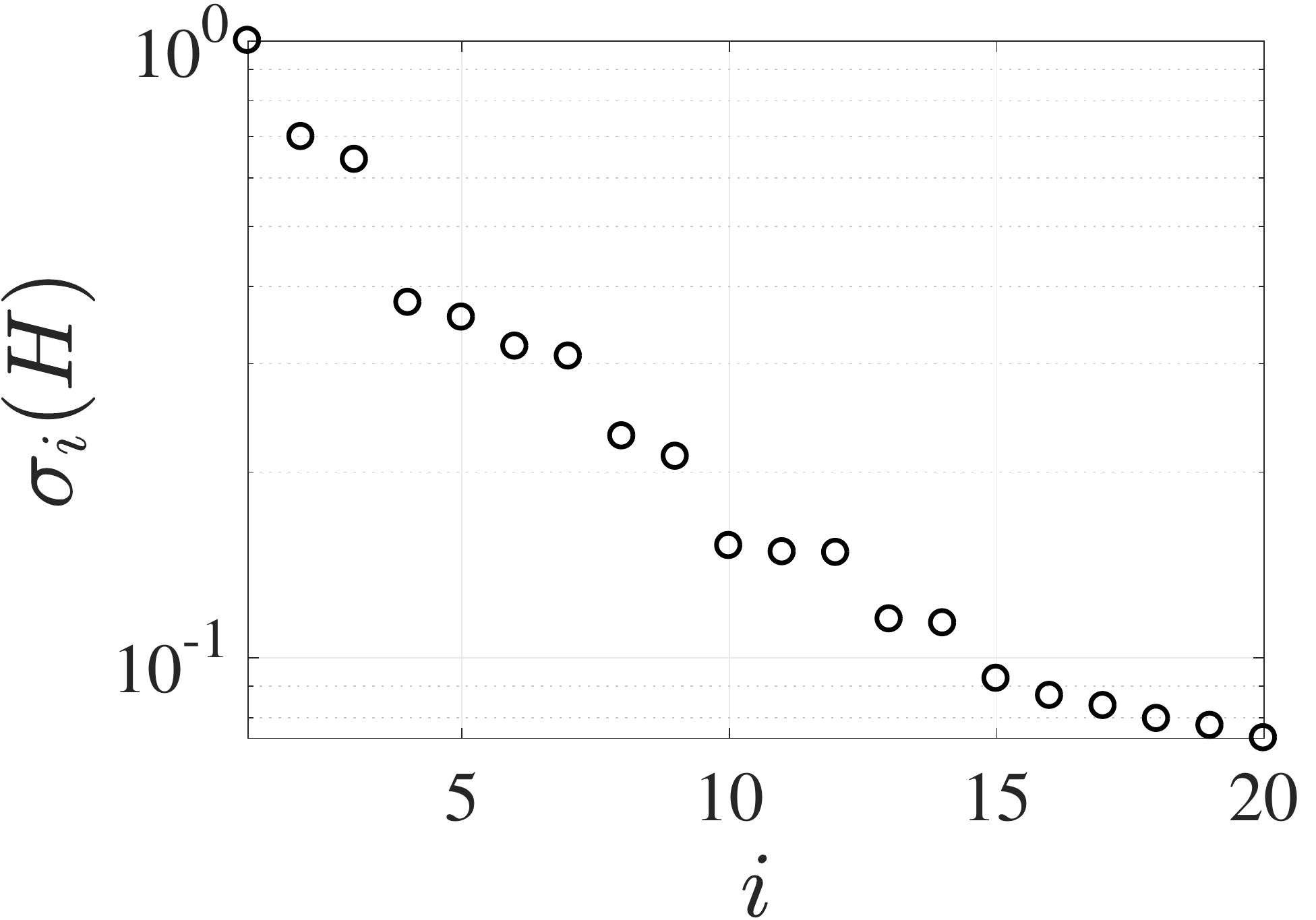}
}
{
\includegraphics[width=0.23\textwidth]{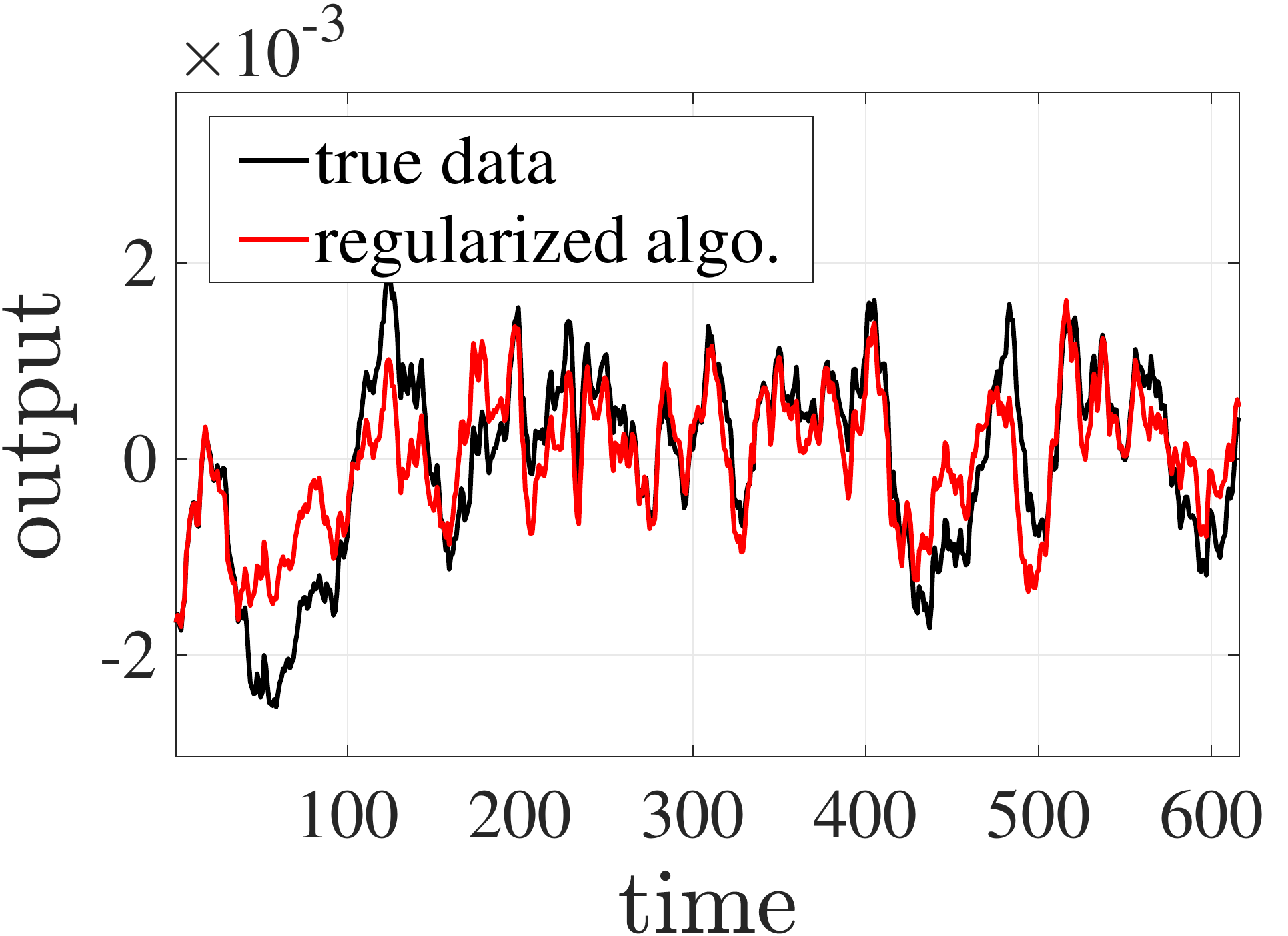}
}\hspace{-0.5em}
{
\includegraphics[width=0.23\textwidth]{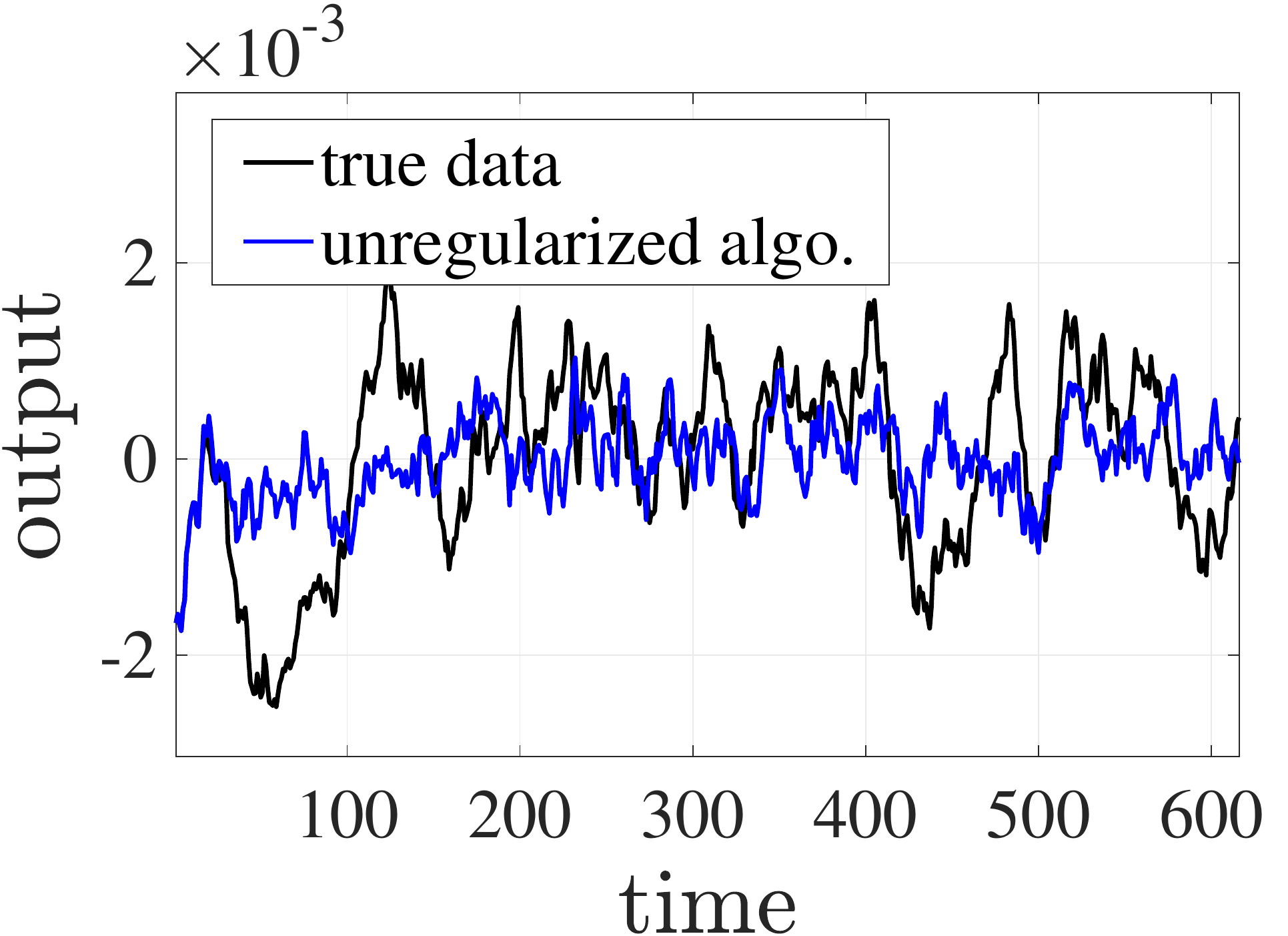}
}
}
{\caption{The first two: Stabilized inverted pendulum data, singular values of the Hankel regularized algorithm and least squares. The last two: Recovery by Hankel regularized algorithm and least squares when Hankel matrix is $40\times40$. Training data size is $16$ and validation data size is $600$.} \label{fig:bigr_inv_pend_eig}}
\end{figure}

\textbf{Large data regime: Both Hankel and least-squares algorithms work well (Fig. \ref{fig:error_bigr_CD}).} The first two figures in Fig. \ref{fig:error_bigr_CD} show the training and validation errors of Hankel-regularized and unregularized methods with hyperparameters $\lambda$ and $n$ respectively. We then choose the best system by tuning the hyperparameters to achieve the smallest validation error. The third figure in Fig.~\ref{fig:error_bigr_CD} plots the training and validation output sequence of the dataset for these algorithms. We see that with sufficient sample size, the system is recovered well. However, the validation error is more flat as a function of $1/\lambda$ (first figure) whereas it is sensitive to the choice of $n$ (second figure), thus $\lambda$ is easier to tune compared to $n$. 

\textbf{Small data regime: Hankel-regularization succeeds while least-squares may fail due to overfitting (Fig. \ref{fig:bigr_CD_eig}).}
The first two figures in Fig.~\ref{fig:bigr_CD_eig} show that the Hankel spectrums of the two algorithms have a notable difference: The system recovered by Hankel-regularization is low-order and has larger singular value gap. The last two figures in Fig.~\ref{fig:bigr_CD_eig} show the advantage of regularization with much better validation performance. As expected from our theory, the difference is most visible in small sample size (this experiment uses 50 training samples). When the number of observations $T$ is small, Hankel-regularization still returns a solution close to the true system while least-squares cannot recover the system properly.

\textbf{Learning a linear approximation of a nonlinear system with few data (Fig.~\ref{fig:bigr_inv_pend_eig}).} Finally, we show that Hankel-regularization can identify a stable nonlinear system via its linearized approximation as well. We consider the inverted pendulum as the experimental environment. First we use a linearized controller to stabilize the system around the equilibrium, and apply single rollout input to the closed-loop system, which is i.i.d.~random input of dimension $1$. The dimension of the state is $4$, and we observe the output of dimension $1$, which is the displacement of the system. We then use the Hankel-regularization and least-squares to estimate the closed-loop system with a linear system model and predict the trajectory using the estimated impulse response. We use $T = 16$ observations for training, and set the dimension to $n = 45$. Fig.~\ref{fig:bigr_inv_pend_eig} shows the singular values and estimated trajectory of these two methods. Despite the nonlinearity of the ground-truth system, the regularized algorithm finds a linear model with order $6$ and the predicted output has small error, while the correct order is not visible in the singular value spectrum of the unregularized least-squares.

\section{Future directions}
This paper established new sample complexity and estimation error bounds for system identification. We showed that nuclear norm penalization works well with small sample size regardless of the misspecification of the problem (i.e.~fitting impulse response with a much larger length rather than the true order). For least-squares, we provided the first guarantee that is optimal in sample complexity and the Hankel spectral norm error. These results can be refined in several directions. In the proof of Theorem \ref{main_gauss}, we used a weighted Hankel operator. We expect that directly computing the Gaussian width of the original Hankel operator will also lead to improvements over least-squares. We also hope to extend the results to account for single trajectory analysis and process noise. In both cases, an accurate analysis of the regularized problem will likely lead to new algorithmic insights.

\bibliographystyle{IEEEtran}
\bibliography{Bibfiles}
\appendix
\section{Sample Complexity for MISO and MIMO Problems} 
This section establishes sample complexity bounds for MISO and MIMO systems. The technical argument builds on \cite{cai2016robust} and extends their results from from SISO case to MISO. We consider recovering a MISO system impulse response. The system is given in \eqref{eq:linear_system}), with output size $m=1$ and the system is order $R$. For multi-rollout case, we only observe the output at time $2n-1$, and let $u_{2n}=0$, we have
\begin{align}\label{eq:mimo_y}
    y_{2n-1} = \sum_{i=1}^{2n-2} CA^{2n-2-i}B u_i + Du_{2n-1}.
\end{align}
 Denote the impulse response by $h\in\mathbb{R}^{p(2n-1)}$, which is a block vector
\begin{align*}
    h = \begin{bmatrix}
    h^{(1)\top}&
    h^{(2)\top}&
    ...
    h^{(2n-1)\top}
    \end{bmatrix}^\top
\end{align*}
where each block $h^{(i)}\in \mathbb{R}^{p}$. $\beta\in\mathbb{R}^{p(2n-1)}$ is a weighted version of $h$, with weights $\beta^{(i)} = K_i h^{(i)}$ and $\beta = \begin{bmatrix}
    \beta^{(1)\top}&
    \beta^{(2)\top}&
    ...&
    \beta^{(2n-1)\top}
    \end{bmatrix}^\top$.
Define the reweighted Hankel map for the same $h$ by
\begin{align*}
    \mathcal{G}(\beta) = \begin{bmatrix}
    \beta^{(1)}/K_1 & \beta^{(2)}/K_2 & \beta^{(3)}/K_3 & ...\\
    \beta^{(2)}/K_2 & \beta^{(3)}/K_3 & \beta^{(4)}/K_2 & ...\\
    ...
    \end{bmatrix}^\top\in\mathbb{R}^{n\times pn}
\end{align*}
and $\mathcal{G}^*$ is the adjoint of $\mathcal{G}$. We define each rollout input $u_1,...,u_{2n-1}$ as independent Gaussian vectors with 
\begin{align}
    u_i \sim \mathcal{N}(0,K_i^2 {\bf I})\label{eq:weight_input}
\end{align}
Now let $\Ub\in \mathbb{R}^{T\times p(2n-1)}$, each entry is iid standard Gaussian. Let $y\in \mathbb{R}^{T}$ be the concatenation of outputs
\begin{align*}
    y = \begin{bmatrix}
    y_1^\top &
    y_2^\top &
    ... &
    y_T^\top
    \end{bmatrix}^\top
\end{align*}
where $y_i\in\R^m$ is defined in \eqref{eq:mimo_y}. We consider the question
\begin{equation}\label{miso_problem}
    \begin{split}
        \min_{\beta'}&\quad \|\mathcal{G}(\beta')\|_*\\
        \mbox{s.t.,}&\quad \|\Ub \beta' - y\|_2 \le \delta
    \end{split}
\end{equation}
where the norm of overall (state and output) noise is bounded by $\delta$. We will present the following theorem, which generalizes the result of \cite{cai2016robust} from SISO case to MISO case.
\begin{theorem}\label{thm:miso_mimo}
Let $\beta$ be the true impulse response. If $T = \Omega((\sqrt{pR}\log n + \epsilon)^2)$ is the number of output observations, $C$ is some constant, the solution $\hat\beta$ to \eqref{miso_problem} satisfies $\|\beta - \hat\beta\|_2\le 2\delta/\epsilon$ with probability 
\begin{align*}
    1- \exp\left(-\frac{1}{2}(\sqrt{T-1} - C(\sqrt{pR}\log n + \epsilon) - \epsilon)^2 \right).
\end{align*}
When the system output is $y = \Ub \beta +  z$ and $z$ is i.i.d. Gaussian noise with variance $\sigma_z^2$, we have that $\|\beta - \hat\beta\|_2\lesssim (\sqrt{pR + \epsilon})  \sigma_z\log n$ with probability (\cite[Thm 1]{oymak2013simple}) 
\begin{align*}
    1- 6\exp\left(-\frac{1}{2}(\sqrt{T-1} - C(\sqrt{pR}\log n + \epsilon) - \epsilon)^2 \right).
\end{align*}
\end{theorem}
This theorem says that when the input dimension is $p$, the sample complexity is $O(\sqrt{pR}\log n)$. The proof strongly depends on the following lemma \cite{cai2016robust, gordon1988milman}:
\begin{lemma}\label{lem:Gaussian_width_prob}
Define the Gaussian width
\begin{align*}
    w(S):= E_g (\sup_{\gamma\in S} \gamma^\top g)
\end{align*}
where $g$ is standard Gaussian vector of size $p$. The Gaussian width of the normal cone of \eqref{eq:lasso_prob} and \eqref{miso_problem} are different up to a constant \cite{banerjee2014estimation}. Let $\Phi=\mathcal{I}(\beta)\cap \mathbb{S}$ where $\mathbb{S}$ is unit sphere. We have 
\begin{align*}
    P( \min_{z\in \Phi}\|\Ub z\|_2 < \epsilon )
    \le  \exp\left(-\frac{1}{2}(\sqrt{T-1} - w(\Phi) - \epsilon)^2 \right) .
\end{align*}
\end{lemma}

We will present the proof in Appendix \ref{s:miso_mimo}.

{\bf MIMO.} For MIMO case, we say output size is $m$. We take each channel of output as a system of at most order $R$, and solve $m$ problems
\begin{align*}
    \mtx{P}_i: \min_{\beta_i}&\quad \|\mathcal{G}(\beta_i)\|_*\\
    \mbox{s.t.,}&\quad \|\Ub x_i - y_i\|_2 \le \delta,\\
    y_i\in\mathbb{R}^T &\mbox{ is the $i$th output.}
\end{align*}
and for each problem we have failure probability equal to \eqref{prob}, which means the total failure probability is 
\begin{align*}
    m\exp\left(-\frac{1}{2}(\sqrt{T-1} - w(\Phi) - \epsilon)^2 \right)
\end{align*}
so we need $T = O((\sqrt{pR}\log n+\log(m) + \epsilon)^2 )$. Let the solution to those optimization problems be $[x_1^*,...,x_m^*]$, and the true impulse response be $[\hat x_1,...,\hat x_m]$, then $\|[x_1^*,...,x_m^*]-[\hat x_1,...,\hat x_m]\|_F\le \sqrt{m}\delta/\epsilon$ with probability 
\begin{align*}
    1- \exp\left(-\frac{1}{2}(\sqrt{T-1} - w(\Phi) - \epsilon)^2 \right)
\end{align*}

We propose another way of MIMO system identification. For each rollout of input data, the output is $m$ dimensional, but we take $1$ channel of output from the observation and throw away other $m-1$ output. And we uniformly pick among channels and get $T$ observations for each channel, and in total $mT$ observations/input rollouts. In this case, when the sample complexity is $m\sqrt{pR}\log n$ ($m$ times of before), we can recover the impulse response with Frobenius norm $\sqrt{m}\delta/\epsilon$ with probability 
\begin{align*}
    1- \exp\left(-\frac{1}{2}(\sqrt{T/m-1} - w(\Phi) - \epsilon)^2 \right)
\end{align*}
\subsection{Gaussian Width of Regularization Problem with MISO and MIMO System}\label{s:miso_mimo}
The following sequence of theorems/lemmas are similar to \cite{cai2016robust}, and we present them here for the integrity of the proof.
\begin{theorem}\label{thm:miso_mimo_2}
Let $\beta$ be the true impulse response. If $T = \Omega((\sqrt{pR}\log(n) + \epsilon)^2)$ is the number of output observations, $C$ is some constant, the solution $\hat\beta$ to \eqref{miso_problem} satisfies $\|\beta - \hat\beta\|_2\le 2\delta/\epsilon$ with probability 
\begin{align*}
    1- \exp\left(-\frac{1}{2}(\sqrt{T-1} - C(\sqrt{pR}\log(n) + \epsilon) - \epsilon)^2 \right).
\end{align*}
\end{theorem}
\begin{proof}
Let $\mathcal{I}(\beta)$ be the descent cone of $\|\mathcal{G}(\beta)\|_*$ at $\beta$, we have the following lemma:
\begin{lemma}\label{err}
Assume 
\begin{align*}
   \min_{z\in \mathcal{I}(\beta)}\frac{\|\Ub z\|_2}{\|z\|_2} \ge \epsilon,
\end{align*}
then $\|\beta - \hat\beta\|_2\le 2\delta/\epsilon$.
\end{lemma}
This is proven in \cite[Lemma 1]{cai2016robust}. To prove Theorem \ref{thm:miso_mimo_2}, we only need lower bound the left hand side with Lemma \ref{err}. The following lemma gives the probability that the left hand side is lower bounded. 

\begin{lemma}\label{lem:Gaussian_width_prob_2}
Define the Gaussian width 
\begin{align}
    w(S):= E_g (\sup_{\gamma\in S} \gamma^Tg)\label{eq:gauss_wid_G}
\end{align}
where $g$ is standard Gaussian vector of size $p$. Let $\Phi=\mathcal{I}(\beta)\cap \mathbb{S}$ where $\mathbb{S}$ is unit sphere. We have 
\begin{align}\label{prob}
    P( \min_{z\in \Phi}\|\Ub z\|_2 < \epsilon )
    \le  \exp\left(-\frac{1}{2}(\sqrt{T-1} - w(\Phi) - \epsilon)^2 \right) .
\end{align}
\end{lemma}
Now we need to study $w(\Phi)$. 
\begin{lemma}\label{lem:dist}
(\cite[Eq.(17)]{cai2016robust}) Let $\mathcal{I}^*(\beta)$ be the dual cone of $\mathcal{I}(\beta)$, then 
\begin{align}
    w(\Phi) \le E(\min_{\gamma\in \mathcal{I}^*(\beta)} \|g - \gamma\|_2).
\end{align}\label{eq:gauss_wid_G_cvx}
\end{lemma}

\vspace{-1em} Note that $\mathcal{I}^*(\beta)$ is just the cone of subgradient of $\mathcal{G}(\beta)$, so it can be written as 
\begin{align*}
    \mathcal{I}^*(\beta) = \{ \mathcal{G}^*(V_1V_2^T + W) |V_1^TW = 0, WV_2=0, \|W\|\le 1 \}
\end{align*}
where $\mathcal{G}(\beta) = V_1\Sigma V_2^T$ is the SVD of $\mathcal{G}(\beta)$. So 
\begin{align*}
    \min_{\gamma\in \mathcal{I}^*(\hat x)} \|g - \gamma\|_2 = \min_{\lambda,W} \| \lambda \mathcal{G}^*(V_1V_2^T + W) - g \|_2.
\end{align*}
For right hand side, we have 
\begin{align*}
    &\quad \| \lambda \mathcal{G}^*(V_1V_2^T + W) - g \|_2\\
    &= \|\lambda \mathcal{G}\mathcal{G}^*(V_1V_2^T + W) - \mathcal{G}(g)\|_F\\
    &= \|\lambda (V_1V_2^T + W) - \mathcal{G}(g)\|_F + \|\lambda (I - \mathcal{G}\mathcal{G}^*)(V_1V_2^T + W)\|_F\\
    &\le \|\lambda (V_1V_2^T + W) - \mathcal{G}(g)\|_F.
\end{align*}
Let $\mathcal{P}_W$ be projection operator onto subspace spanned by $W$, i.e., 
\begin{align*}
   \{ W |V_1^TW = 0, WV_2=0 \}
\end{align*}
and $\mathcal{P}_V$ be projection onto its orthogonal complement. 

We have
\begin{align*}
    w(\Phi) & \le E(\min_{\lambda,W} \| \lambda \mathcal{G}^*(V_1V_2^T + W) - g \|_2).
\end{align*}
We will upper bound it by the specific choice $\lambda = \|\mathcal{P}_W(\mathcal{G}(g))\|$, $W = \mathcal{P}_W(\mathcal{G}(g))/\lambda$. 
\begin{align*}
    &\quad \|\lambda (V_1V_2^T + W) - \mathcal{G}(g)\|_F\\
    &=
    \|\mathcal{G}(g) - \mathcal{P}_W(\mathcal{G}(g)) - \|\mathcal{P}_W(\mathcal{G}(g))\|V_1V_2^T\|_F\\
    &\le \|\mathcal{P}_V(\mathcal{G}(g)) - \|\mathcal{P}_W(\mathcal{G}(g))\|V_1V_2^T\|_F\\
    &\le \|\mathcal{P}_V(\mathcal{G}(g))\|_F + \|\mathcal{P}_W(\mathcal{G}(g))\|\|V_1V_2^T\|_F\\
    &= \|\mathcal{P}_V(\mathcal{G}(g))\|_F + \sqrt{R} \|\mathcal{P}_W(\mathcal{G}(g))\|\\
    &= \|\mathcal{P}_V(\mathcal{G}(g))\|_F + \sqrt{R} \|\mathcal{G}(g)\|.
\end{align*}
Bound the first term by (note $V_1$ and $V_2$ span $R$ dimensional space, so $V_1\in\mathbb{R}^{n\times R}$ and $V_2\in\mathbb{R}^{pn\times R}$)
\begin{align*}
    \|\mathcal{P}_V(\mathcal{G}(g))\|_F &= 
    \|V_1V_1^T\mathcal{G}(g) + (I - V_1V_1^T)\mathcal{G}(g)V_2V_2^T\|_F\\
    &\le \|V_1V_1^T\mathcal{G}(g)\|_F + \|\mathcal{G}(g)V_2V_2^T\|_F\\
    &\le 2\sqrt{R} \|\mathcal{G}(g)\|.
\end{align*}
So we get
\begin{align*}
    w(\Phi) \le 3\sqrt{R}\|\mathcal{G}(g)\|.
\end{align*}
We know that, if $p=1$, then $E\|\mathcal{G}(g)\| = O(\log(n))$. For general $p$, let 
\begin{align*}
    g^{(i)} = [g^{(i)}_1,...,g^{(i)}_p]^T,
\end{align*}
we rearrange the matrix as 
\begin{align*}
    \bar{\mathcal{G}}(g) = [G_1,...,G_p]
\end{align*}
where 
\begin{align*}
    G_i =  \begin{bmatrix} g^{(1)}_i& g^{(2)}_i/\sqrt{2} & ... \\
    g^{(2)}_i/\sqrt{2} & g^{(3)}_i/\sqrt{3} & ...\\
    ...\end{bmatrix}
\end{align*}
and the expectation of operator norm of each block is $\log(n)$. Then (note $v$ below also has a block structure $[v^{(1)};...;v^{(n)}]$)
\begin{align*}
    \|\bar{\mathcal{G}}(g)\|
    &= \max_{u,v} \frac{u^T\bar{\mathcal{G}}(g)v}{\|u\|\|v\|}
    = \max_{u,v^1,...,v^p} \sum_{i=1}^{p} \frac{u^TG_iv^{(i)}}{\|u\|\|v\|}\\
    &\le \max_{v^1,...,v^p}  O(\log(n))\frac{\sum_{i=1}^{p}\|v^{(i)}\|}{\sqrt{\sum_{i=1}^{p}\|v^{(i)}\|^2}}\\
    &\le O(\sqrt{p}\log(n)).
\end{align*}

$\|\bar{\mathcal{G}}(g)\| = \|\mathcal{G}(g)\|$. So we have $\|\mathcal{G}(g)\| = \sqrt{p}\log(n)$. So $w(\Phi) = C\sqrt{pR}\log(n)$. Get back to \eqref{prob}, we want the probability be smaller than 1, and we get 
\begin{align*}
    \sqrt{T-1} - C\sqrt{pR}\log n - \epsilon > 0
\end{align*}
thus $T = O((\sqrt{pR}\log(n) + \epsilon)^2)$. 
\end{proof}
At the end, we give a different version of Theorem \ref{thm:miso_mimo_2}. Theorem \ref{thm:miso_mimo_2} in \cite{cai2016robust} works for the any noise with bounded norm. Here we consider the iid Gaussian noise, and use the result in \cite{oymak2013simple}, we have the following theorem.
\begin{theorem}\label{thm:weiyu_random}
Let the system output $y = \Ub \beta +  z$ where $\Ub$ entries are iid Gaussian $\mathcal{N}(0,1/T)$, $\beta$ is the true system parameter and $ z\sim\mathcal{N}(0,\sigma_z^2)$. Then \eqref{miso_problem} recovers $\hat \beta$ with error $\|\hat \beta - \beta\|_2 \le w(\Phi)\| z\|_2 / \sqrt{T}  \lesssim \sqrt{pR}\sigma_z\log n$ with high probability. 
\end{theorem}
\begin{remark}
Since the power of $\Ub$ is $n$ times of that of $\bar \Ub$ and the variance of $\Ub$ is $1/T$, $\sigma_z = \sqrt{n/T}\sigma$, we have $\|\hat h - h\|_2 \le \|\hat \beta - \beta\|_2\lesssim \sqrt{\frac{pnR}{T}}\sigma\log n$.
\end{remark}
\section{Proof of Regularization Algorithm's Spectral Norm Error (Thm. \ref{main_gauss})}
We will prove the first case of \eqref{eq:reg_spec_error}. The second case is a direct application of \cite{cai2016robust}.

\begin{theorem}\label{main_supp}
We study the problem 
\begin{align}\label{eq:lasso_prob app}
    \min_{\beta^\prime}&\quad  \frac{1}{2}\|\Ub \hat \beta^\prime - y\|^2 + \lambda\|\mathcal{G}(\hat \beta^\prime)\|_*,
\end{align}
in the MISO setting ($m$=1, $p$ inputs), where $\Ub\in\mathbb{R}^{T\times (2n-1)p}$. Let $\beta$ denote the (weighted) impulse response of the true system which has order $R$, i.e., $\rank(\mathcal{G}(\beta))=R$, and let $y = \Ub \beta + \xi$ be the measured output, where $\xi$ is the measurement noise. 
Finally, denote the minimizer of \eqref{eq:lasso_prob app} by $\hat \beta$.
Define 
\begin{align*}
    \mathcal{J}(\beta) &:= \left\{ v \ \big|\ \langle v, \partial (\frac{1}{2}\|\Ub^\top\beta - y\|^2 + \lambda\|\cG(\beta)\|_*)\rangle\le 0 \right\},\\
    \Gamma &:= \|I - \Ub^\top\Ub\|_{\mathcal{J}(\beta)},
\end{align*}
$\mathcal{J}(\beta)$ is the normal cone at $\beta$, and $\Gamma$ is the spectral RSV. If $\Gamma<1$, $\hat \beta$ satisfies
\begin{align*}
    \|\cG(\hat \beta - \beta)\| \le \frac{\|\cG(\Ub^\top\xi)\| + \lambda}{1-\Gamma}. 
\end{align*}
\end{theorem}
\begin{lemma}\label{rank_reduction}
Suppose $\xi\sim \mathcal{N}(0,\sigma_\xi I)$, $T \lesssim pR^2\log^2 n$, and $\Ub$ has iid Gaussian entries with $\mathbf{E}(\Ub^\top  \Ub) = 1$. Then, we have that $\mathbf{E}(\Gamma)<0.5$, and $P(\Gamma<0.5)\geq 1-O(R\log n \sqrt{p/ T})$. In this case $ \|\cG(\hat \beta - \beta)\| \lesssim \sigma_\xi \sqrt{p}\log n$.
\end{lemma}
\begin{remark}
To be consistent with the main theorem in the paper, we need to find the relation between $\sigma_\xi$ and SNR, or $\sigma$.
We do the following computation: (1) $\cG(\hat \beta - \beta) = \cH(\hat h - h)$, so we are bounding the Hankel spectral norm error here; (2) Each column of the input is unit norm, so each input is $\mathcal{N}(0,1/T)$, and the average power of input is $1/T$; (3) Because of the scaling matrix $K$, the actual input of $\bar \Ub$ is $n$ times the power of entries in $\Ub$. With all above discussion, we have $\sigma_\xi = \sigma\sqrt{n/T}$, which results in $ \|\cG(\hat \beta - \beta)\| \lesssim  \sqrt{\frac{np}{T}}\sigma\log n$.
\end{remark}
\begin{proof}
Now we bound $\|\cG(\hat \beta - \beta)\|$ by partitioning it to $\|\cG(I - \Ub^\top\Ub)(\hat \beta - \beta)\|$ and $\|\cG (\Ub^\top\Ub(\hat \beta - \beta))\|$. We have 

\begin{equation}\label{eq:lasso_1}  
\begin{split}
     &\quad \|\cG(I - \Ub^\top\Ub)(\hat \beta - \beta)\| \\
     &= \|\cG(I - \Ub^\top\Ub)\cG^*\cG(\hat \beta - \beta)\|\\
    &\le \|\cG(I - \Ub^\top\Ub)\cG^*\|_{2,\mathcal{GJ}(\beta)} \|\cG(\hat \beta - \beta)\|\\
    & = \Gamma\|\cG(\hat \beta - \beta)\|.
    \end{split}
\end{equation}
And then we also have 
\begin{align*}
  \|\cG(\Ub^\top\Ub(\hat \beta - \beta))\| = \|\cG\Ub^\top(\Ub\hat \beta - y + \xi)\| \le \|\cG\Ub^\top(\Ub\hat \beta - y)\| + \|\cG(\Ub^\top\xi)\|.
\end{align*}
Since $\hat \beta$ is the optimizer, we have $\Ub^\top(\Ub\hat \beta - y) + \lambda\cG^*(\hat V_1\hat V_2^\top + \hat W) = 0$ where $\cG(\hat \beta) = \hat V_1\hat \Sigma\hat V_2^\top $ is the SVD of $\cG(\hat \beta)$, $\hat W\in\mathbb{R}^{n\times n}$ where $\hat V_1^\top\hat W = 0, \hat W\hat V_2=0, \|\hat W\|\le 1$. 
We have 
\begin{align}\label{eq:lasso_2}
   \|\cG(\Ub^\top\Ub(\hat \beta - \beta))\| \le \|\cG(\Ub^\top\xi)\|+\lambda.
\end{align}

Combining \eqref{eq:lasso_1} and \eqref{eq:lasso_2}, we have 
\begin{align*}
\|\cG(\hat \beta - \beta)\| 
    &\le \|\cG(I - \Ub^\top\Ub)(\hat \beta - \beta)\| +  \|\cG(\Ub^\top\Ub(\hat \beta - \beta))\|\\
    &\le 
    \Gamma\|\cG(\hat \beta - \beta)\| + \|\cG(\Ub^\top\xi)\|+\lambda
\end{align*}
or equivalently,$
    \|\cG(\hat \beta - \beta)\| \le \frac{\|\cG(\Ub^\top\xi)\|+\lambda}{1-\Gamma},
    $ and $\Gamma = \|\cG(I - \Ub^\top\Ub)\cG^*\|_{2,\mathcal{GJ}(\beta)}.$

\textbf{Bounding $\Gamma$.} Denote the SVD of $\cG(\beta) = V_1\Sigma V_2^\top$. Denote projection operators $\mathcal{P}_V(M) = V_1V_1^\top M + MV_2V_2^\top - V_1V_1^\top MV_2V_2^\top$ and $\mathcal{P}_W(M) = M - \mathcal{P}_V(M)$. First we prove a result for later use. 
\begin{align}
    &\frac{1}{2}\|y - \Ub \hat \beta\|^2 + \lambda\|\cG \hat \beta\|_* \le \frac{1}{2}\|y - \Ub \beta\|^2 + \lambda\|\cG \beta\|_* \notag\\
    \Rightarrow\quad & \frac{1}{2}\|y - \Ub \hat \beta\|^2 + \lambda\|\cG \hat \beta\|_* \le \frac{1}{2}\|\xi\|^2 + \lambda\|\cG \beta\|_* \notag \\
    \Rightarrow\quad & \frac{1}{2}\|\Ub \beta + \xi - \Ub \hat \beta\|^2 + \lambda\|\cG \hat \beta\|_* \le \frac{1}{2}\|\xi\|^2 + \lambda\|\cG \beta\|_* \notag \\
    \Rightarrow\quad&\frac{1}{2}\|\Ub (\beta - \hat \beta)\|^2 + \xi^\top\Ub(\beta - \hat \beta) + \lambda\|\cG \hat \beta\|_* \le \lambda\|\cG \beta\|_* \notag \\
    \Rightarrow\quad& \lambda\|\cG \hat \beta\|_* \le \lambda\|\cG \beta\|_* + \xi^\top\Ub(\hat \beta -  \beta) \notag \\
    \Rightarrow\quad& \|\cG \hat \beta\|_* - \|\cG \beta\|_* \le \frac{\|\cG(\Ub^\top\xi)\|}{\lambda} \|\cG(\hat \beta -  \beta)\|_*\label{eq:lasso_tangent}
    \end{align}
    \eqref{eq:lasso_tangent} is an important result to note, and following that,
    \begin{align}
    & \|\cG \hat \beta\|_* - \|\cG \beta\|_* \le \frac{\|\cG(\Ub^\top\xi)\|}{\lambda} \|\cG(\hat \beta -  \beta)\|_*\notag\\
    \Rightarrow\quad& \langle \cG(\hat \beta - \beta), V_1V_2^\top+W\rangle \le \frac{\|\cG(\Ub^\top\xi)\|}{\lambda} \|\cG(\hat \beta -  \beta)\|_*\notag\\
    \Rightarrow\quad& \|\cP_W\cG(\hat \beta - \beta)\|_* \le -\langle \cG(\hat \beta - \beta), V_1V_2^\top\rangle + \frac{\|\cG(\Ub^\top\xi)\|}{\lambda} \|\cG(\hat \beta -  \beta)\|_* \notag\\
    \Rightarrow\quad& \|\cP_W\cG(\hat \beta - \beta)\|_* \le \|\cP_V\cG(\hat \beta - \beta) \|_* + \frac{\|\cG(\Ub^\top\xi)\|}{\lambda} (\|\cP_V\cG(\hat \beta -  \beta)\|_* + \|\cP_{W}\cG(\hat \beta -  \beta)\|_*)\notag\\
    \Rightarrow\quad& \|\cP_{W}\cG(\hat \beta -  \beta)\|_* \le \frac{1 + \frac{\|\cG(\Ub^\top\xi)\|}{\lambda}}{1 - \frac{\|\cG(\Ub^\top\xi)\|}{\lambda}}\|\cP_V\cG(\hat \beta -  \beta)\|_*\label{eq:lasso_approx_low_rank}
\end{align}
Let $\Ub$ be iid Gaussian matrix with scaling $\mtx{E}(\Ub^\top\Ub) = I$. Here we need to study the Gaussian width of the normal cone $w(\mathcal{J}(\beta))$ of \eqref{eq:lasso_prob app}. \cite{banerjee2014estimation} proves that, if \eqref{eq:lasso_tangent} is true, and $\lambda\ge 2\|\cG(\Ub^\top\xi)\|$, then the Gaussian width of this set (intersecting with unit ball) is less than $3$ times of Gaussian width of $\{\hat \beta: \|\cG(\hat \beta)\|_* \le \|\cG(\beta)\|_* \}$, which is $O(\sqrt{R}\log n)$ \cite{cai2016robust}.\newline
A simple bound is that, let $\delta = \hat \beta - \beta$, $\Gamma$ can be replaced by $\max \|\cG((I - \Ub^\top\Ub)\delta)\| / \|\cG(\delta)\|$
subject to $\hat \beta\in\mathcal{J}(\beta)$. With \eqref{eq:lasso_approx_low_rank}, we have $\|\cP_{W}\cG(\delta)\|_* \le 3\|\cP_V\cG(\delta)\|_*$.

Denote $\sigma = \|\cG(\delta)\|$, we know that
\begin{align*}
    \sigma &\ge \max\{\|\cP_{W}\cG(\delta)\|, \|\cP_{V}\cG(\delta)\|\},\\ 
    \|\cP_{V}\cG(\delta)\|&\ge \|\cP_{V}\cG(\delta)\|_*/(2R).
\end{align*}And simple algebra gives that $\max_{0<\sigma_i<\sigma, \sum_i\sigma = S} \sum_i \sigma_i^2 \le S\sigma$. 
So let $\sigma_i$ be singular values of $\cP_{V}\cG(\delta)$ or $\cP_{W}\cG(\delta)$, and $S = \|\cP_{V}\cG(\delta)\|_*$ or $\|\cP_{W}\cG(\delta)\|_*$,
\begin{align*}
    \frac{\sigma}{\|\cP_{V}\cG(\delta)\|_F} &\ge
    \sqrt{\frac{\|\cP_{V}\cG(\delta)\|_*}{2R\|\cP_{V}\cG(\delta)\|_*}} \ge \sqrt{1/2R},\\
     \frac{\sigma}{\|\cP_{W}\cG(\delta)\|_F} &\ge \sqrt{\frac{\|\cP_{V}\cG(\delta)\|_*}{2R\|\cP_{W}\cG(\delta)\|_*}} 
     \ge \sqrt{1/6R}
\end{align*}
the second last inequality comes from \eqref{eq:lasso_approx_low_rank}. Thus if $\|(I - \Ub^\top\Ub)\delta\| = O(1/\sqrt{R}) \|\delta\|$, in other words, $\|\cG((I - \Ub^\top\Ub)\delta)\|_F = O(1/\sqrt{R}) \|\cG(\delta)\|_F$, whenever $\delta$ in normal cone, we have
\begin{align}
    &\quad \|\cG((I - \Ub^\top\Ub)\delta)\| \le  \|\cG((I - \Ub^\top\Ub)\delta)\|_F\notag\\
    &\le O(1/\sqrt{R})  \|\cG(\delta)\|_F \le \|\cG(\delta)\|.\label{eq:spec_relation}
\end{align}
so $\Gamma<1$. To get this, we need $\sqrt{T}/w(\mathcal{J}(\beta)) \gtrsim \sqrt{R}$ where $T \gtrsim pR^2\log^2 n$ \cite[Thm 9.1.1]{vershynin2018high}, still not tight in $R$, but $O(\min\{n, R^2\log^2n\})$ is as good as \cite{oymak2018non} and better than \cite{sarkar2019finite}, which are $O(n)$ and $O(n^2)$ correspondingly. \cite[Thm 9.1.1]{vershynin2018high} is a bound in expectation, but it naively turns into high probability bound since $\Gamma\ge 0$. 
\end{proof} 

\section{Proof of Suboptimal Recovery Guarantee with IID Input (Thm. \ref{thm:counter_iid})}

We consider the Gaussian width $w(\Phi)$ defined in this specific case.

Let $V = \frac{1}{n} \ones\ones^\top$, and $\mathcal{I}^*(h) = \{ \mathcal{H}^*(V + W) |\ones^\top W = 0, W\ones=0, \|W\|\le 1 \}$, we have\footnote{We slightly change the definition of Gaussian width. We refer readers to \cite[Thm 1]{mccoy2013achievable}. It is known to be as tight and the probability of failure is order constant if the number of measurements is smaller than order square of the quantity.} 
$w(\Phi) = E( \min_{\lambda, W} \|\lambda \mathcal{H}^*(V+W)-g \|_2)$. In the instance, $V = \frac{1}{n}\ones\ones^\top$. and we take $W$ such that $\|W\|\le 1$ and $W\ones = W^\top\ones = 0$. 

First, we note that 
\begin{align}
    &\quad E( \min_{\lambda, W} \|\lambda \mathcal{H}^*(V+W)-g \|_2)\notag \\
    &= \frac{1}{2}\left(E( \min_{\lambda, W} \|\lambda \mathcal{H}^*(V+W)-g \|_2\ |\ \ones^\top g\le 0) + E( \min_{\lambda, W} \|\lambda \mathcal{H}^*(V+W)-g \|_2\ |\ \ones^\top g> 0)\right)\notag\\
    &\ge \frac{1}{2}E( \min_{\lambda, W} \|\lambda \mathcal{H}^*(V+W)-g \|_2\ |\ \ones^\top g\le 0).\label{eq:counter_gauss_w}
\end{align}

\emph{Proof strategy: Based on the previous derivation, we focus on the case when $\ones^\top  g \le 0$. Denote $z = \lambda\cH^*(V + W) - g$, and the vector $z_{1:k}$ is the first $1$ to $k$ entries of $z$. Then we prove that $(1)\ \lambda \le \|z\|_2/\sqrt{n},\quad 
    (2)\ \|z_{1:1/\lambda}\|_2 \gtrsim\lambda^{-1/2}.$
Then we have 
\begin{align*}
    \|z\|_2 \ge \|z_{1:1/\lambda}\|_2 \gtrsim\lambda^{-1/2} \gtrsim(\|z\|_2/\sqrt{n})^{-1/2}
\end{align*}
which suggests $\|z\|_2 \gtrsim n^{1/6}$.
}
\begin{lemma}
Let $g$ be a standard Gaussian vector of size $2n-1$ conditioned on $\ones^\top g \le 0$. Let $z = \lambda\cH^*(V + W) - g$ where $V = \frac{1}{n}\ones\ones^\top$, and $W\ones = W^\top\ones = 0$, $\|W\| \le 1$. Then we have that  $\lambda \le \|z\|_2/\sqrt{n}$.
\end{lemma}
We observe that $\ones^\top\cH^*(X)$ is the summation of every entry in $X$ for any matrix $X$. Thus $\ones^\top\cH^*(W) = 0$ since $W\ones = 0$. Conditioned on $\ones^\top g\le 0$, we have 
\begin{align*}
    \ones^\top(\lambda\cH^*(V + W) - g) \ge \lambda\ones^\top\cH^*(V) = \lambda n.
\end{align*}
And so that $\|\lambda\cH^*(V + W) - g\|_2 \ge \lambda\sqrt{n}$. Then $\|z\|_2/\sqrt{n} \ge \lambda$, we have proven the first point. 
\begin{lemma}
Let $g$ be a standard Gaussian vector of size $2n-1$ conditioned on $\ones^\top g \le 0$. Let $z = \lambda\cH^*(V + W) - g$ where $V = \frac{1}{n}\ones\ones^\top$, and $W\ones = W^\top\ones = 0$, $\|W\| \le 1$. Let  the vector $z_{1:k}$ is the first $1$ to $k$ entries of $z$.Then we have that $\|z_{1:1/\lambda}\|_2 \gtrsim \lambda^{-1/2}$.
\end{lemma}

If $\|z\|_2\le \sqrt{n}$, we observe $z_{1:[\sqrt{n}/\|z\|_2]}$, where $[\sqrt{n}/\|z\|_2]$ is the integer part of $\sqrt{n}/\|z\|_2$, and $z_{1:[\sqrt{n}/\|z\|_2]}$ is the subset of $z$ containing the entries from index $1$ to $[\sqrt{n}/\|z\|_2]$. When $i\le \sqrt{n}/\|z\|_2$, the $i$-th entry of $\cH^*(V + W)$, denoted as $(\cH^*(V + W))_i$, is summation of $2i$ terns in $V$ and $W$. Since these two matrices have bounded spectral norm $1$, then every entry of $V$ is $1/n$ and every entry of $W$ is no bigger than $1$. So
\begin{align*}
    z_i &= \lambda(\cH^*(V + W))_i - g_i
    \in \pm (1+1/n)i\lambda - g_i\\
    &\in \pm \frac{(1+1/n)i\|z\|_2}{\sqrt{n}} - g_i.
\end{align*}
We denote $[1,...,\sqrt{n}/\|z\|_2]\in\R^n$ as the vector whose $i$th entry is $\sqrt{i}/\|z\|_2$, and get
\begin{align*}
    \|z_{1:\sqrt{n}/\|z\|_2}\|_2 
    &\ge \|g_{1:\sqrt{n}/\|z\|_2}\|_2 
     - \frac{(1+1/n)\|z\|_2}{\sqrt{n}} \|[1,...,\sqrt{n}/\|z\|_2]\|_2  \\
    &\ge \frac{n^{1/4}}{\|z\|_2^{1/2}} -\frac{(1+1/n)n^{1/4}}{\sqrt{3}\|z\|_2^{1/2}}.
\end{align*}
Note that the first term is smaller than the second term, so we have $\|z_{1:\sqrt{n}/\|z\|_2}\|_2 \ge C_1\frac{n^{1/4}}{\|z\|_2^{1/2}}$ for some constant $C_1$. Note this is the norm of a part of $z$, which is smaller than the norm of $z$, so we have $\frac{C_1n^{1/4}}{\|z\|_2^{1/2}} \le \|z\|_2$.
So that $\|z\|_2 \gtrsim n^{1/6}$, and we have bounded the quantity \eqref{eq:counter_gauss_w}.

\section{Proof of Least Squares' Spectral Norm Error (Thm. \ref{thm ls spectral})}\label{s:ls_append}
\begin{theorem}\label{thm:ls app}
    Denote the discrete Fourier transform matrix by $F$. Denote $ z_{(i)}\in\R^T, i=1,...,m$ as the noise that corresponds to each dimension of output. The solution $\hat h$ of 
    \begin{align}\label{eq:generic_ls app}
        \hat{h}:=h + \bar{\Ub}^\dag z=\min_{h'}~  \frac{1}{2}\|\bar{\Ub} h' - y\|_F^2.
    \end{align}
    obeys 
   $\|\hat h - h\|_F \le \| z\|_F / \sigma_{\min}(\bar{\Ub})$, and
        $\|\mathcal{H}(\hat h - h)\| \le \left\| \left[\|F\bar{\Ub}^\dag z_{(1)}\|_\infty, ..., \|F\bar{\Ub}^\dag z_{(m)}\|_\infty\right]\right\|$.
\end{theorem}
\begin{proof} 
First we clarify the notation here. In regularization part, we only consider the MISO system, whereas we can prove the bound for MIMO system as well in least square. Here we assume the input is $p$ dimension and output is $m$ dimension, at each time. For the notation in \eqref{eq:generic_ls app}, $\bar{\Ub} \in \R^{T\times (2n-1)p}$, whose each row is the input in a time interval of length $2n-1$. The impulse response is $h\in\R^{(2n-1)p\times m}$ and output and noise are $y, z\in\R^{T\times m}$, where each column corresponds to one channel of the output. Each row of $y$ is an output observation at a single time point. $ z_{(i)}\in\R^T$ is a column of the noise, meaning one channel of the noise contaminating all observations at this channel.

Eq.\eqref{eq:generic_ls app} has close form solution and we have $\|\hat h - h\| = \|\bar{\Ub}^\dag z\|\le \| z\| / \sigma_{\min}(\bar{\Ub})$. To get the error bound in Hankel matrix, we can denote $\bar  z = \bar{\Ub}^\dag z = (\bar{\Ub}^\top\bar{\Ub})^{-1}\bar{\Ub}^\top  z$, and 
\begin{align*}
    H_{\bar  z} = \begin{bmatrix}
    \bar  z_1& \bar  z_2&...&\bar  z_{2n-1}\\
    \bar  z_2&\bar  z_3 &...&\bar  z_1\\
    ...\\
    \bar  z_{2n-1}& \bar  z_1 & ... & \bar  z_{2n-2}
    \end{bmatrix}.
\end{align*}
If $m=1$, $\bar  z\in\R^{(2n-1)p}$ is a vector \cite[Sec. 4]{krahmer2014suprema} proves that $H_{\bar  z} = F^{-1}\diag{F\bar  z}F$. So the spectral norm error is bounded by $\|\diag{F\bar  z}\|_2 = \|F\bar  z\|_\infty$. 

If $m>1$, all columns of $ z$ are independent, so $H_{\bar  z}$ can be seen as concatenation of $m$ independent noise matrices where each satisfies the previous derivation. 
\end{proof}
Next we prove Thm. \ref{thm ls spectral}.
\begin{proof} We use Theorem \ref{thm:ls app}. First let $m=1$. The covariance of $F\bar  z = F\bar{\Ub}^\dag  z$ is  $F(\bar{\Ub}^\top\bar{\Ub})^{-1}F^\top$. If $T \gtrsim n$, it's proven in \cite{vershynin2018high} that $\frac{TI}{2}\preceq \bar{\Ub}^\top\bar{\Ub}\preceq \frac{3TI}{2}$. Then $\frac{n}{2T}I\preceq F(\bar{\Ub}^\top\Ub)^{-1}F^\top \preceq \frac{3n}{2T}I$. So $\|F\bar  z\|_\infty$ should scale as $O(\sigma_z\sqrt{\frac{n}{T}}\log n)$, and then $\|\mathcal{H}(\bar  z)\|_2\le \|H_{\bar  z}\|_2 \le \|F\bar  z\|_\infty = O(\sigma_z\sqrt{\frac{n}{T}}\log n)$. 

If $m>1$, then by concatenation we simply bound the spectral norm by $m$ times MISO case. When $m>1$, with previous discussion of concatenation, and each submatrix to be concatenated has the same distribution, so the spectral norm error is at most $\sqrt{m}$ times larger.
\end{proof}
\section{Proof of End to End Bound of System Identification and Model Selection (Thm. \ref{thm:e2e})}\label{sec: e2e hankel}

We select $\delta>0$ such that $T_\mathrm{val} \gtrsim  \frac{1}{\delta^2} \log \frac{|\Lambda|}{P}$, and denote $a_1 = 1 - \frac{\delta}{\delta+2},\ a_2 = 1 + \frac{\delta}{\delta+2}$. Then we have $a_2/a_1 = 1+\delta$. Let $T_0 = \max\{1, T/(T_{\mathrm{val}}R\log^2n)\}$. We will show that 
\begin{align} 
 &\frac{\|\hat{h} - h\|_2}{\sqrt{2}} \leq \|\cH(\hat{h} - h)\| \lesssim   
 \begin{cases} (1 + T_0^{1/4})\frac{a_2}{a_1} \sqrt{\frac{np}{\snr\times T}}\log(n),\ \text{if}\ T\gtrsim \min(R^2,n)\\
 (1 + T_0^{1/4})\frac{a_2}{a_1}\sqrt{\frac{Rnp}{\snr\times T}}\log(n),\ \text{if}\ R\lesssim T\lesssim \min(R^2,n).\end{cases}\label{eq:reg_spec_error_model_select_app2}
\end{align}
Note that we will need $T_0^{1/4}\delta \lesssim 1$ from our choice of $T_\mathrm{val}$ in the theorem, so the bound is sufficient for the theorem. This will be used later to calculate $\delta$ in \eqref{eq: Tval low bd}.

We use the change of variable as in \eqref{eq:lasso_prob}. We learn the parameter $\beta$ with different $\lambda$, and get different estimations $\hat \beta$ which is a function of $\lambda$. To be more explicit, let $\hat \beta(\lambda)$ be the estimator associated with a certain regularization parameter $\lambda$. Among all the estimators, we choose the solution with the smallest validation error, which is denoted as 
    \begin{align*}
        \hat \beta = \argmin_{\hat \beta(\lambda)} \| \Ub_{\mathrm{val}} \hat \beta(\lambda) - y_{\mathrm{val}}\|_2^2 
    \end{align*}
    Denote the noise in validation data as $\xi_{\mathrm{val}}$. We have that
    \begin{align}
        &\quad \| \Ub_{\mathrm{val}} \hat \beta - y_{\mathrm{val}}\|_2^2 = \|\Ub_{\mathrm{val}} (\hat \beta - \beta) - \xi_{\mathrm{val}}\|_2^2\notag\\
        &= \|\Ub_{\mathrm{val}} (\hat \beta - \beta)\|_2^2 + \|\xi_{\mathrm{val}}\|_2^2 - 2\xi_{\mathrm{val}}^\top \Ub_{\mathrm{val}} (\hat \beta - \beta). \label{eq:val_0}
    \end{align}
    In this formulation, $\|\xi_{\mathrm{val}}\|_2^2$ in \eqref{eq:val_0} is regarded as fixed among all validation instances, and we study the other two terms. Since $\Ub_{\mathrm{val}}$ is normalized that each entry is i.i.d. $\mathcal{N}(0,1/T_{\mathrm{val}})$, we have $\mtx{E}\|\Ub_{\mathrm{val}} (\hat \beta - \beta)\|_2^2=\|\hat \beta - \beta\|_2^2$. 
    
    The quantity $\xi_{\mathrm{val}}^\top \Ub_{\mathrm{val}} (\hat \beta - \beta)$ is zero mean and we know that
        $\Ub_{\mathrm{val}} (\hat \beta - \beta) \sim \mathcal{N} (0, \frac{\|\hat \beta - \beta\|_2^2}{T_{\mathrm{val}}} I)$. Thus the variance of $\xi_{\mathrm{val}}^\top \Ub_{\mathrm{val}} (\hat \beta - \beta)$ is bounded by  $O(\sigma_{\xi_{\mathrm{val}}}^2\|\hat \beta - \beta\|_2^2 / T_{\mathrm{val}})$ (the distribution of the inner product is sub-exponential). We know that
    \begin{align*}
        \|\hat \beta - \beta\|_2 \approx \sqrt{\frac{R\log^2 n}{T}} \|\xi\|_2 =  \sqrt{\frac{R\log^2 n}{T}} \sqrt{T_{\mathrm{val}}}\sigma_{\xi_{\mathrm{val}}}. 
    \end{align*}

    \textbf{Case 1:} If $T_{\mathrm{val}} \gtrsim \frac{T}{R\log^2(n)}$, we have that $ \|\hat \beta - \beta\|_2 \gtrsim \sigma_{\xi_{\mathrm{val}}}$. 
    
    Suppose the number of validated parameters $\lambda$ is $|\Lambda|$ and we need to choose the size of validation data. With different validation data size $T_\mathrm{val}$, the variance of $\|\Ub_{\mathrm{val}} (\hat \beta - \beta)\|_2^2$ decreases with rate $1/T_\mathrm{val}$. 
    
    We fix factors $a_1,a_2$, such that with high probability, for all choices of $\lambda$, $\|\Ub_{\mathrm{val}} (\hat \beta - \beta)\|_2^2  - 2\xi_{\mathrm{val}}^\top \Ub_{\mathrm{val}} (\hat \beta - \beta)$ is in the set $(a_1 \|\hat \beta - \beta\|_2^2,a_2 \|\hat \beta - \beta\|_2^2)$. We know that:
    the terms $\|\Ub_{\mathrm{val}} (\hat \beta - \beta)\|_2^2$  and   $2\xi_{\mathrm{val}}^\top \Ub_{\mathrm{val}} (\hat \beta - \beta)$ are subexponential;
    The mean of $\|\Ub_{\mathrm{val}} (\hat \beta - \beta)\|_2^2$ is $\|\hat \beta - \beta\|_2^2$ and the variance is $O(\|\hat \beta - \beta\|_2^4/T_\mathrm{val})$;
    The mean of $2\xi_{\mathrm{val}}^\top \Ub_{\mathrm{val}} (\hat \beta - \beta)$ is $0$ and the variance is $O(\|\hat \beta - \beta\|_2^4/T_\mathrm{val})$ (Note that $ \|\hat \beta - \beta\|_2 \gtrsim \sigma_{\xi_{\mathrm{val}}}$ in this case).
    
    By Bernstein bound \cite[Prop. 5.16]{vershynin2010introduction}, we know that the probability that the quantity of \eqref{eq:fail_select} is not between $(a_1,a_2)\cdot \|\hat \beta - \beta\|_2^2$ is $\exp(-\min_i (a_i-1)^2T_{\mathrm{val}})$ where $(a_i-1)^2\approx \delta^2$. 
    
    Hence there exists a constant $c$ such that for every choice of $\lambda$, 
    \begin{align}
        &\mathbf{Pr} \left( \big| \|\Ub_{\mathrm{val}} (\hat \beta - \beta)\|_2^2  - 2\xi_{\mathrm{val}}^\top \Ub_{\mathrm{val}} (\hat \beta - \beta)\big| \notin (a_1,a_2)\cdot \|\hat \beta - \beta\|_2^2 \right) < \exp(-c\delta^2 T_{\mathrm{val}} ). \label{eq:fail_select}
    \end{align}
    
    We choose probability $P$ that any of the event in \eqref{eq:fail_select} happens. If all $|\Lambda|$ validations corresponding to $\lambda_i$ succeed, then we use the union bound on \eqref{eq:fail_select} and solve for $|\Lambda| \exp(-c\delta^2 T_{\mathrm{val}} )<P$. 
    Thus we set $T_\mathrm{val} = \max\{\frac{T}{R\log^2(n)}, \frac{1}{c\delta^2} \log \frac{|\Lambda|}{P}\} \label{eq:T_val}$.
    so that \eqref{eq:reg_spec_error_model_select} holds with probability $1-P$.
    
    \textbf{Case 2:} If $T_{\mathrm{val}} \lesssim \frac{T}{R\log^2(n)}$, then we denote $T_0 = T/(T_{\mathrm{val}}R\log^2n)$, with similar derivation as above, we know that the mean of $\|\Ub_{\mathrm{val}} (\hat \beta - \beta)\|_2^2$ is $\|\hat \beta - \beta\|_2^2$ and the variance is $O(\|\hat \beta - \beta\|_2^4/T_\mathrm{val})$; The mean of $2\xi_{\mathrm{val}}^\top \Ub_{\mathrm{val}} (\hat \beta - \beta)$ is $0$ and the variance is $ O(T_0\|\hat \beta - \beta\|_2^4/T_\mathrm{val})$. Thus, similar to \eqref{eq:fail_select}, 
    \begin{align*}
        &\mathbf{Pr} \left( \big| \|\Ub_{\mathrm{val}} (\hat \beta - \beta)\|_2^2  - 2\xi_{\mathrm{val}}^\top \Ub_{\mathrm{val}} (\hat \beta - \beta)\big| \notin (a_1,a_2)\cdot \sqrt{T_0}\|\hat \beta - \beta\|_2^2 \right) < \exp(-c\delta^2 T_{\mathrm{val}} ). 
    \end{align*}
    The following steps are same as the first case, and the error is multiplied by $T_0^{1/4}$ compared to the first case.
    
    At the end, we will need to argue about the lower bound for $T_{\mathrm{val}}$. We used two inequalities in the proof above:
    \begin{align*}
        T_{\mathrm{val}} \gtrsim  \frac{1}{\delta^2} \log( \frac{|\Lambda|}{P}),\ T_0^{1/4} \delta \lesssim 1. 
    \end{align*}
    They are equivalent to 
    \begin{align}
        T_{\mathrm{val}} &\gtrsim  \frac{1}{\delta^2} \log (\frac{|\Lambda|}{P}), \
        T_{\mathrm{val}} \gtrsim \frac{\delta^4T}{R\log^2(n)}. \label{eq: Tval low bd}
    \end{align}
    Setting the right hand side to be equal, we have 
    \begin{align*}
        \delta^2 = \left(T^{-1} \log(\frac{|\Lambda|}{P}) R\log^2(n)\right)^{1/3}.
    \end{align*}
    Plugging it into any lower bound for $T_{\mathrm{val}}$ in \eqref{eq: Tval low bd}, we get the bound in the main theorem.
    
\textbf{Comparison of sample complexity of regularized algorithm and unregularized least squares: model selection with data being requested online.} Algorithm \ref{algo:1} uses static data for training and validation, which means that, the total $T + T_{\mathrm{val}}$ samples are given and fixed, and we split the data and run Algorithm \ref{algo:1}. We denote the total sample complexity $\Ttot = T + T_{\mathrm{val}}$.
To be fully efficient in sample complexity, we can start from $\Ttot=0$, keep requesting new samples, which means increasing $\Ttot$, and run Algorithm \ref{algo:1} for each $\Ttot$. When the validation error is small enough (which happens when $\Ttot\gtrsim R$), we know the algorithm recovers a impulse response estimation with the error in Theorem \ref{thm:e2e} and we can \textbf{terminate} the algorithm.\\
We compare it with the model selection algorithm in \cite{sarkar2019finite} for least squares estimator, and we find that it does \textbf{not terminate} until $\Ttot\gtrsim n$. For least squares, the parameter to be tuned is the dimension of the variable, i.e., we vary the length of estimated impulse response. We call the tuning variable $\ntune$ and it is upper bounded by $n$. We keep increasing $\Ttot$ and train by varying $\ntune\in[1,\Ttot/2]$ (so the least squares problems are overdetermined). The output $y$ is collected at time $2\ntune-1$. We consider two impulse responses truncated at length $n$: $h^1 = \ones_n$ (order $=1$) and $h^{n_1} = [\ones_{n_1}; \zeros_{n-n_1}]$ (order $=n_1$). As long as $\Ttot< n_1$, one cannot differentiate $h^1$ and $h^{n_1}$, because $y$ is collected at time $\Ttot$ and the $\Ttot+1$-th to the $n$-th terms of $h^1,h^{n_1}$ do not contribute to $y$. Even if the system is order $1$, one does not know it and cannot terminate the algorithm. Thus the tuning algorithm in  \cite{sarkar2019finite} requires $\Ttot \gtrsim n$. This does not happen with regularization, because we always collect $y$ at time $n$ in Algorithm \ref{algo:1}, but not at time $2\ntune-1$, thus the algorithm always detects the difference between $h^1,h^{n_1}$ after time $n_1$.

\end{document}